\documentclass[conference]{IEEEtran}

\def\comments{1}
\def\anonymous{0}

\usepackage{preamble}

    \usepackage{breqn}
    \makeatletter
    \newcommand{\removelatexerror}{\let\@latex@error\@gobble}
    \makeatother

\makeatletter
\newcommand*{\algrule}[1][\algorithmicindent]{%
  \makebox[#1][l]{%
    \hspace*{.2em}%
  }
}

\newcount\ALG@printindent@tempcnta
\def\ALG@printindent{%
    \ifnum \theALG@nested>0
        \ifx\ALG@text\ALG@x@notext
        \else 
            \unskip
            \ALG@printindent@tempcnta=1
            \loop
            \algrule[\csname ALG@ind@\the\ALG@printindent@tempcnta\endcsname]%
            \advance \ALG@printindent@tempcnta 1
            \ifnum \ALG@printindent@tempcnta<\numexpr\theALG@nested+1\relax
                \repeat
        \fi
    \fi
}
\patchcmd{\ALG@doentity}{\noindent\hskip\ALG@tlm}{\ALG@printindent}{}{\errmessage{failed to patch}}
\patchcmd{\ALG@doentity}{\item[]\nointerlineskip}{}{}{} %
\makeatother

\newcommand\Algphase[1]{%
\vspace{0.1in}
\Statex\hspace*{-\algorithmicindent}\textbf{#1}%
\vspace{0.01in}
}

\usepackage{cleveref}  %

\title{Tukey Depth Mechanisms for Practical Private Mean Estimation
\thanks{Thank you to our generous sponsors.}
}

    \ifnum\anonymous=1
    \author{}
    \else
    \author{\IEEEauthorblockN{Gavin Brown}
    \IEEEauthorblockA{\textit{Paul G. Allen School of Computer Science and Engineering} \\
    \textit{University of Washington}\\
    Seattle, USA \\
    grbrown@cs.washington.edu}
    \and
    \IEEEauthorblockN{Lydia Zakynthinou}
    \IEEEauthorblockA{\textit{Simons Institute for the Theory of Computing} \\
    \textit{University of California, Berkeley}\\
    Berkeley, USA \\
    lydiazak@berkeley.edu}
    }

\begin{document}

\maketitle

\begin{abstract}
Mean estimation is a fundamental task in statistics and a focus within differentially private statistical estimation. While univariate methods based on the Gaussian mechanism are widely used in practice, more advanced techniques such as the exponential mechanism over quantiles offer robustness and improved performance, especially for small sample sizes. Tukey depth mechanisms 
carry
these advantages to multivariate data, 
providing similar strong theoretical guarantees. However, practical implementations fall behind these theoretical developments. 

In this work, we take the first step to bridge this gap by implementing the (Restricted) Tukey Depth Mechanism, a theoretically optimal mean estimator for multivariate Gaussian distributions, yielding improved practical  methods for private mean estimation. Our implementations enable the use of these mechanisms for small sample sizes or low-dimensional data. Additionally, we implement variants of these mechanisms that use approximate versions of Tukey depth, trading off accuracy for faster computation. We demonstrate their efficiency in practice, showing that they are viable options for modest dimensions. Given their strong accuracy and robustness guarantees, we contend that they are competitive approaches for mean estimation in this regime. We explore future directions for improving the computational efficiency of these algorithms by leveraging fast polytope volume approximation techniques, paving the way for more accurate private mean estimation in higher dimensions.\footnote{Code available at \url{https://github.com/gavinrbrown1/private-tukey}.}
\end{abstract}

\section{Introduction}

Mean estimation is one of the most ubiquitous tasks in statistics.
As a result, it is a focus within differentially private~\cite{DworkMNS06} statistical estimation, with a rapidly growing body of work. 
The central test case for differentially private mean estimators is that of Gaussian distributions $\cN(\mu,\Sigma)$.
The complexity of this task under privacy is well-studied from a theoretical perspective. 
For univariate data, there are a number of practical implementations.
The most basic, such as the Laplace or Gaussian mechanisms, follow the ``clip-and-noise'' recipe.
Even better, algorithms such as the exponential mechanism over quantiles have stronger theoretical guarantees and perform well even at small samples sizes.
On multivariate data, however, practice lags significantly behind theory.
In this work, we take an algorithm from the theoretical privacy literature and implement it, yielding a practical private method for low dimensions that has strong robustness and small-sample performance.
We propose modifications to trade off computation and accuracy as the dimension grows.
Finally, we chart a path to faster computation in higher dimensions and provide technical tools to that end.
To do this, we unite several lines of research from theory and practice on differential privacy, robust statistics, and volume computation.

Theoretically, the best-known differentially private algorithms for multivariate Gaussian distributions come from a line of work drawing on robust statistics.
Coupled with work on lower bounds, they establish a nearly tight characterization of the number of samples needed for this task.
For a distribution $\cN(\mu,\Sigma)$, consider the error of an estimate $\hat\mu$ in the \emph{Mahalanobis norm}:
\[
    \norm{\hat\mu-\mu}{\Sigma} \defeq \norm{\Sigma^{-1/2}(\hat\mu - \mu)}{2}.
\]
This norm tightly captures the uncertainty of the empirical mean and represents an error guarantee which is invariant under affine transformations of the data.
The algorithm with the lowest error for this task comes from \cite{brown2021covariance}; we sketch an outline in Algorithm~\ref{alg:discrete_tukey_vol_overview}.
This algorithm is $(\eps,\delta)$-differentially private and, given $n$ independent samples from $\mc{N}(\mu,\Sigma)$, produces $\tilde\mu$ such that $\norm{\tilde\mu-\mu}{\Sigma}\le \alpha$ as long as
\begin{equation}\label{eq:bgsuz_tukey_error}
    n\gtrsim \frac{d}{\alpha^2} + \frac{d}{\alpha\eps} + \frac{\log 1/\delta}{\eps}.
\end{equation}
Here ``$\gtrsim$'' hides (modest) absolute constants and an additive $\frac{\log(1/\alpha \eps)}{\alpha \eps}$ term; the error guarantee holds with high constant probability.
As we discuss in \Cref{app:related-work}, a number of algorithms achieve a similar flavor of guarantee; all are based on techniques from robust statistics and all require exponential or high polynomial running time, the latter achieved through the sum-of-squares paradigm.
The sample complexity in~\eqref{eq:bgsuz_tukey_error} is nearly tight, matching known lower bounds up to constants and the suppressed additive term \cite{karwa2017finite,kamath_KLSU19,kamath2022new}.

Among practical methods for multivariate data, the Gaussian mechanism and its kin reign supreme.
In its canonical form, one \emph{clips} (i.e., projects) the data to an $\ell_2$ ball of radius $R$, computes the empirical mean, and adds appropriately scaled spherical Gaussian noise.
This is $(\eps,\delta)$-differentially private and, on distributions with identity covariance where the clipping has little impact, produces an estimate $\tilde\mu$ such that $\norm{\tilde\mu - \mu}{2}\le \alpha$ as long as
\[
    n \gtrsim \frac{d}{\alpha^2} + \frac{R \sqrt{d\log 1/\delta}}{\alpha \eps}.
\] 
Here ``$\gtrsim$'' hides an absolute constant; the error guarantee holds with high constant probability.
Compared to~\eqref{eq:bgsuz_tukey_error}, we have sacrificed the Mahalanobis error guarantee and also, for $R\approx \sqrt{d}$, suffer an extra factor of $\sqrt{\log (1/\delta)}$, which may be a high price for small sample sizes.
Even achieving $R\approx \sqrt{d}$ requires strong prior knowledge about the location of the distribution, which may be unavailable.
A number of recent works can be seen as providing sophisticated clipping algorithms.
These may improve the dependence on hyperparameters such as $R$ \cite{biswas2020coinpress,HuangLY21, tsfadia2022friendlycore, aumuller2023plan,dagan2024dimension} or give error guarantees other than $\ell_2$ \cite{brown2023fast,duchi2023fast}, but the increase in complexity increases the error by (sometimes large) constant factors.

In this work, we take the first step to bring the theoretically best-known differentially private algorithms for multivariate Gaussian mean estimation to the realm of practical approaches. The estimator achieving the optimal rate in~\eqref{eq:bgsuz_tukey_error} is called the \emph{Restricted Tukey Depth Mechanism}, introduced in the 2021 work of Brown, Gaboardi, Smith, Ullman, and Zakynthinou (henceforth BGSUZ)~\cite{brown2021covariance}.  

The Restricted Tukey Depth Mechanism has several desirable properties: it is robust, it achieves optimal accuracy (up to logarithmic terms) with respect to the tight, affine-invariant Mahalanobis error metric, and it does not require prior knowledge about the parameters of the distribution. (See~\Cref{app:related-work} for related work on differentially private Gaussian mean estimation and comparisons.) 

\subsection{Background: The (Restricted) Tukey Depth Mechanism}
The \emph{Tukey depth} \cite{tukey1975mathematics} of a point $y\in \mathbb{R}^d$ with respect to a dataset $x\in \mathbb{R}^{n\times d}$ is a classic notion of outlyingness in multivariate data:
\begin{equation}
T_x(y) =  \cdot\min_{v\in\mathbb{R}^d}  \Big|\big  \{x_i \in x : \langle x_i, v \rangle \ge \langle y, v\rangle\big \} \Big|,
    \label{eq:tukey_depth_def}
\end{equation}
It generalizes the notion of univariate quantiles and has a long history within robust statistics.
As we can see in~\Cref{fig:tukey_example}, points closer to the center of the data have higher Tukey depth. 
Any point with maximal depth is called a \emph{Tukey median}: such points are robust and accurate mean estimators for Gaussian distributions.
See \cite{nagy2022halfspace} for a recent focused survey.

\begin{figure}[t]
\centerline{
\includegraphics[width=0.8\columnwidth]{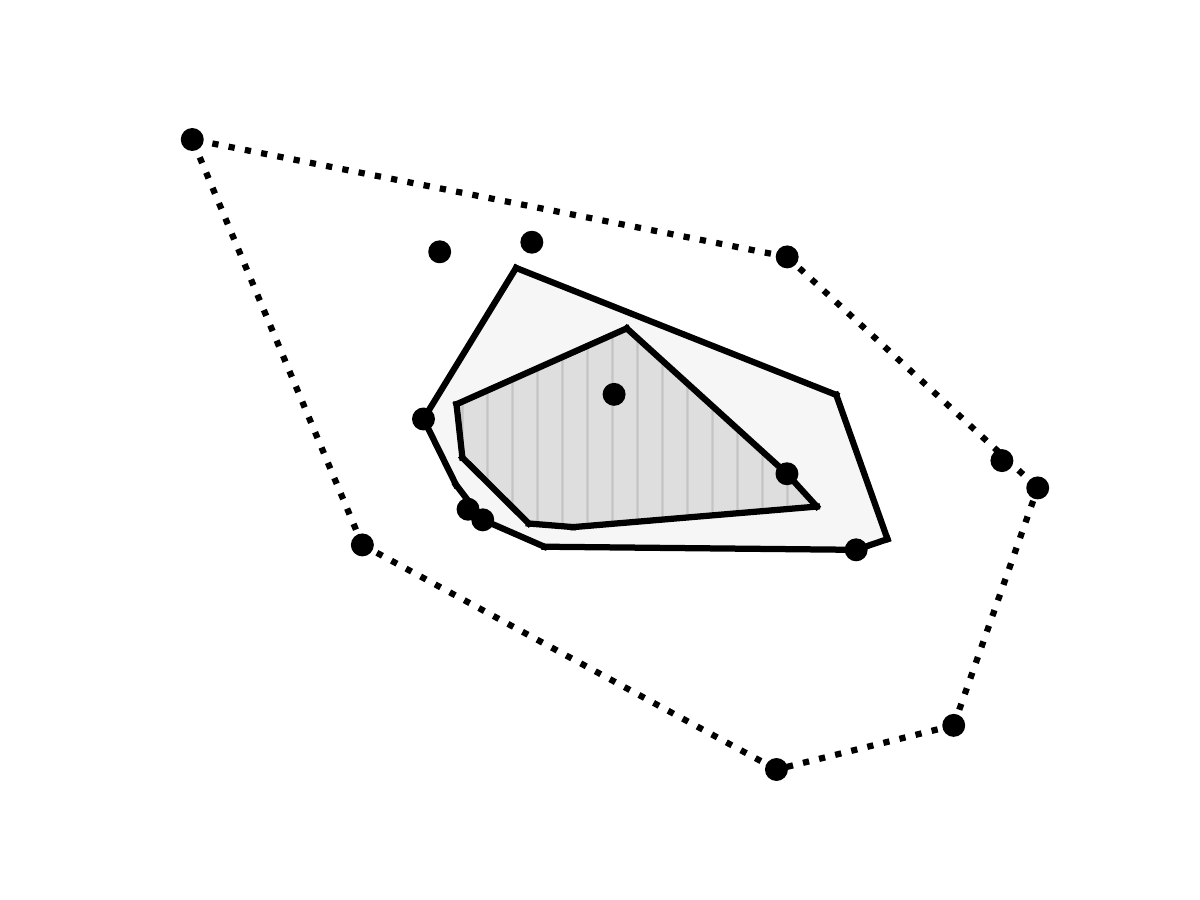}}
\caption{Tukey depth is a multivariate notion of centrality. The convex hull of a dataset $x$ is exactly the set of points $y$ with $T_x(y)>0$, i.e., nonzero depth. The light gray region is the set of points of depth at least 3; inside that is the set of points of depth at least 4 (gray and hatched). Note that Tukey depth is defined for any point in $\mathbb{R}^d$, not just elements of $x$.}
\label{fig:tukey_example}
\end{figure}

If we seek to privately find a point of high Tukey depth, a natural candidate is the \emph{exponential mechanism}~\cite{mcsherry2007mechanism}, since Tukey depth has sensitivity one. 
In its general form, the exponential mechanism produces an output $y\in \mathcal{Y}$ from a distribution proportional to $\exp\{\eps\cdot  q_x(y)\}$, where $\eps$ is a privacy parameter and $q_x(y)$ is a \emph{score} function that depends on the input dataset $x$.

To instantiate this approach with $q_x(y)=T_x(y)$, we require a bounded output space (so that $\int \exp\{\eps\cdot q_x(y)\} \mathrm{d}y$ has a finite integral).
One can achieve this given prior knowledge about the data, for example that it lies in a hypercube $[-R,R]^d$.\footnote{Informally, error bounds for algorithms that satisfy concentrated or pure DP require such assumptions.}
Kaplan, Sharir, and Stemmer~\cite{kaplan2020find} used this approach to solve the problem of finding a point in the convex hull of a dataset.
Liu, Kong, Kakade, and Oh\cite{liu2021robust} use the same approach for robust and pure DP mean estimation.
We will refer to this algorithm as \emph{BoxEM}.

The {\em Restricted Exponential Mechanism} (REM) of BGSUZ offers another solution.
It restricts the mechanism to never return values of $y$ with $q_x(y) < t$ for some threshold $t$.
This modification avoids the need for a parameter bound $R$ but, since the restriction is data-dependent, the standard privacy proof for the exponential mechanism no longer applies.
To work around this, BGSUZ introduce a notion of ``safe'' inputs, defined abstractly as precisely those datasets on which one can run this restricted mechanism while preserving privacy.
They compute the Hamming distance from the input data to the space of ``unsafe'' datasets and apply the propose-test-release (PTR) framework of \cite{DworkL09}. 
This approach requires approximate DP, unlike BoxEM (which satisfies pure DP). The outline of REM is presented in~\Cref{alg:discrete_tukey_vol_overview}. 

\begin{figure}[!t]
\removelatexerror
\begin{algorithm}[H] \caption{REM with Tukey Depth%
}\label{alg:discrete_tukey_vol_overview}
\begin{algorithmic}[1]

\Statex \textit{N.B. This is an algorithmic blueprint; the notion of (Hamming) ``distance to unsafety'' will be specified later.}
\Statex

\Require{Dataset $x = (x_1,\dots,x_{n})^T \in \mathbb{R}^{n\times d}$.
Privacy parameters: $\eps,\delta>0$. Minimum threshold $t$.}

\Algphase{PTR check}

\State $h(x) \gets \mathtt{DistanceToUnsafety}_{(\eps,\delta,t)}(x)$
\If{$h(x) + \mathrm{Lap}(\frac{1}{\eps})< \frac{\log(1/2\delta)}{\eps}$} 
    \Return \texttt{FAIL}. 
\EndIf

\Algphase{Sampling from REM: $\cM_{\eps,t}$}
\State \Return $\tilde\mu$ w.p. $p(y) \propto \begin{cases}
        \exp\left\{\frac{\eps}{2} \cdot  T_x(y)\right\} & \text{if $T_x(y)\ge t$} \\
        0 & \text{otherwise}
        \end{cases}$ 
\end{algorithmic}
\end{algorithm}
\end{figure}

The main drawback of both algorithms is computational inefficiency. 
Computing $T_x(y)$ for even a single point is NP-hard~\cite{johnson1978densest}; the fastest known algorithms take time $n^{d-1}$ and are thus unusable in high dimensions.
Even in low dimensions, it is not immediately clear how to sample efficiently.
Furthermore, as we discuss later, REM's abstract distance-to-unsafety calculation presents an obstacle to implementation even in one dimension.

\subsection{Our Contributions: Practical Implementations of the (Restricted) Tukey Depth Mechanism} We implement and evalutate both REM and BoxEM over Tukey depth.
We also implement variants that use approximate versions of Tukey depth, trading off accuracy for faster computation.
We show their efficiency in practice, establishing that they are viable options in modest dimensions. Given their robustness, affine-invariance, and strong accuracy guarantees, we contend that they are competitive approaches for mean estimation in this regime. 

\begin{enumerate}
\item \textbf{Exact implementation.} Relying on advances by Amin, Joseph, Ribero, and Vassilvitskii~\cite{amin2022easy}, we implement the distance-to-unsafety check, bringing its running time to 
$n^{O(d^2)}$. %
We test this algorithm in 2, 3, and 4 dimensions.

\item \textbf{Random Tukey depth.} 
Cuesta-Albertos and Nieto-Reyes \cite{cuesta2008random} propose the \emph{random Tukey depth}: replace the minimum over $v\in \mathbb{R}^d$ in~\eqref{eq:tukey_depth_def} with a minimum over a set of $k$ vectors chosen randomly from the unit sphere. 
We incorporate this into our algorithms, yielding an asymptotic running time of $n\directions^d$ and a significant speed-up in practice.
\end{enumerate}

In~\Cref{sec:approximate_volume}, we discuss the natural next step of incorporating approximate polytope volume computation as a path to higher dimensions.

\section{Preliminaries}\label{app:preliminaries}

We work with differentially privacy, the de facto standard for privacy protection in data analysis.
Differentially private algorithms are randomized algorithms whose output distributions are indistinguishable on any two adjacent datasets.
To present the definition, we now formalize the notions of \emph{indistinguishable} and \emph{adjacent}.
\begin{defn}[Indistinguishability]
    For any $\eps,\delta\ge 0$, we say that two distributions $p$ and $q$ over a domain $\mc{U}$ are \emph{$(\eps,\delta)$-indistinguishable}, denoted $p\approx_{(\eps,\delta)} q$, if for all measurable $E\subseteq \mc{U}$ we have
    \[
        p(E) \le e^{\eps}q(E) + \delta\quad\text{and}\quad q(E)\le e^{\eps}p(E)+\delta.
    \]
\end{defn}
\begin{defn}[Adjacency]
    For a set $\mc{X}$ and natural number $n$, we say that $x=(x_1,\ldots,x_n)\in\mc{X}^n$ and $x'=(x_1',\ldots,x_n')\in \mc{X}^n$ are \emph{adjacent} if there exists an $i^*$ such that, for all $i\neq i^*$, $x_i=x_i'$.
\end{defn}
\begin{defn}[Differential Privacy]
    A randomized algorithm $A : \mc{X}^n\to \mc{U}$ is \emph{$(\eps,\delta)$-differentially private} if, for all adjacent $x,x'\in\mc{X}^n$, we have $A(x)\approx_{(\eps,\delta)} A(x')$.
\end{defn}
The case of $\delta=0$ is called \emph{pure} differential privacy, while $\delta>0$ is called \emph{approximate}.
Our notion of adjacency is sometimes referred to as ``swap,'' in contrast with ``add/remove.''

The algorithms we implement are provably robust: they will be accurate on data drawn from Gaussian distributions, even if an adversary is allowed to arbitrarily change a constant fraction of the data.
We formalize this corruption model now.
\begin{defn}[Strong Contamination]\label{def:strong-contamination}
    Let $x=(x_1,\ldots,x_n)\in \mc{X}^n$ and $x'=(x_1',\ldots,x_n')\in\mc{X}^n$ be datasets and consider $\alpha\in (0,1)$.
    We say that $x'$ is an \emph{$\alpha$-corruption of $x$} if there exists a set $S\subset [n]$ with $|S|\le \alpha n$ such that, for all $i\notin S$, $x_i=x_i'$.
\end{defn}
Note that this definition allows an adversary to inspect the data $x$ and select which points to replace and how to replace them.

Finally, we present some crucial geometric definitions and facts.
A \emph{halfspace} $H_{(a,b)}$ is a subset of $\mathbb{R}^d$, parameterized by $a\in \mathbb{R}^d$ and $b\in \mathbb{R}$ and defined as all $y$ such that $\langle a, y\rangle\le b$.
A \emph{polytope} $P\subseteq \mathbb{R}^d$ is an intersection of a finite number of halfspaces: we collect $k$ halfspaces $\{(a_i, b_i)\}_{i}$ into a matrix-vector pair $(A,b)$ and write $P = \{y : Ay\le b\}$.
This is called the \emph{H-representation} of $P$.
Alternatively, a polytope can be described as the convex hull of a set of points, or vertices; this is called the \emph{V-representation} of $P$.

For a fixed dataset $x$, we call the set of points of Tukey depth at least $\ell$ the \emph{Tukey upper-level set}, denoted $\mc{Y}_{\ge \ell}$ as the dataset will be clear from context.
The set $\mc{Y}_{\ge \ell}$ is a polytope. 
It is easy to see that it is the intersection of halfspaces: infer from~\eqref{eq:tukey_depth_def} that each direction $v$ defines a halfspace that depends on the quantiles of the projections $\{\langle x_i,v\rangle\}_{i\in[n]}$.
We do not need infinitely many vectors $v$ to describe the intersection, roughly $n^{d-1}$ suffice~\cite{liu2019fast}.
The set $\mc{Y}_{\ge 1}$ is exactly the convex hull of the data.
We will use the notation $\mc{Y}_{=\ell}$ to denote the set of points with depth exactly $\ell$.

\section{Related Work}\label{app:related-work}

There is a great deal of work on differentially private Gaussian mean estimation.
There are many axes on which to compare algorithms, representing tradeoffs between running time, robustness, type of privacy guarantee, and sample complexities.
Algorithms also differ in the style of error guarantee and the assumptions they make about the data distribution.
This section provides a brief overview of this prior work, with emphasis on algorithms that can be implemented. 
We restrict our review to the central model of differential privacy, where the data is held by a single curator.

Karwa and Vadhan~\cite{karwa2017finite} first established approaches for learning univariate Gaussians with optimal error, applying the Gaussian mechanism and stability-based techniques from~\cite{DworkL09}.
For isotropic multivariate data $\mc{N}(\mu,\mathbb{I})$ (or, equivalently, $\mc{N}(\mu,\Sigma)$ with $\Sigma$ known) under the assumption that the true mean has $\ell_2$ norm at most $R$, the naive Gaussian mechanism has error that scales linearly with $R$.
Later papers~\cite{kamath_KLSU19,biswas2020coinpress} give estimators with near-optimal error for this setting, with error depending logarithmically on a priori bounds on the range $R$ of the data.

References~\cite{HuangLY21, tsfadia2022friendlycore} give polynomial-time algorithms for private aggregation, with mean estimation as a central application. 
These 
aim to minimize the dependence of the error on the range of the data. 
In the known-covariance case, these have guarantees comparable to~\cite{biswas2020coinpress}. 
Taking an alternative approach, \cite{aumuller2023plan, dagan2024dimension} tailor the amount of noise dimension-by-dimension to dataset's variance, yielding improved guarantees for low-rank distributions, when the covariance matrix is (almost) known. 

In the unknown-covariance setting, one might aim for error guarantees in Mahalanobis distance.
The concurrent works of \cite{duchi2023fast,brown2023fast} gave the first polynomial-time differentially private algorithms achieving this goal with sample complexity linear in $d$.
Their algorithms' errors have no dependence on the range of the data or condition number.
Both algorithms are implementable but have involved analyses that result in significant constants in their privacy guarantees, yielding poor small-sample performance.

From the above techniques, we experimentally compare with the Gaussian mechanism and with CoinPress \cite{biswas2020coinpress}. 
Huang, Liang, and Yi~\cite{HuangLY21} and Aum{\"u}ller, Lebeda, Nelson, and Pagh~\cite{aumuller2023plan} also release code but, in the setting of our simulations, perform similarly. 
These algorithms, and all described so far, end by introducing Gaussian noise. 

In contrast, our work sits within an established literature connecting differential privacy to robust statistics through the exponential mechanism.\footnote{See  \cite{avella2023differentially,yu2024gaussian} for insightful discussion on an alternative path, privatizing M-estimators through optimization.}
Formal statements of these connections appeared in \cite{DworkL09}, who introduced the \emph{propose-test-release} framework and, among other applications, analyzed the exponential mechanism over quantiles for univariate mean estimation.
This method performs very well in practice; similar ideas drive the DP Theil-Sen estimator for simple linear regression \cite{alabi2020differentially,sarathy2022analyzing}.
See~\Cref{sec:univariate_experiments} for more on univariate estimation.
A number of papers extend these connections, exploring when one can turn robust statistical algorithms into private ones and vice versa \cite{asi2020instance, liu2022differential,georgiev2022induces,hopkins2023robustness,alabi2022privately,asi2023robustness}. 
None of these algorithms seem amenable to implementation.

The Tukey depth mechanisms we consider represent a natural generalization of the univariate exponential mechanism run on quantiles.
Kaplan, Sharir, and Stemmer~\cite{kaplan2020find} applied this approach to the task of producing a point within the convex hull of the input dataset.
Liu, Kong, Kakade, and Oh~\cite{liu2021robust} and BGSUZ applied it to the the unknown covariance case, giving guarantees in Mahalanobis distance under sample complexity depending optimally on the dimension $d$. 
The application of~\cite{liu2021robust} requires prior knowledge of parameter bounds in the form of a bounding box $[-R,+R]^d$.
Its accuracy depends logarithimcally on $R$, whereas the Restricted Tukey Depth Mechanism of BGSUZ is free of any dependence on parameter bounds. 
Both approaches are computationally inefficient. 
Other work on computing depth functions privately includes \cite{ramsay2021differentially,cumings2022differentially,ramsay2023differentially}. 

\paragraph{Properties of the (Restricted) Tukey Depth Mechanism} The Restricted Tukey Depth Mechanism has several desirable properties.
First, it does not require any prior knowledge of parameters of the distribution, such as the covariance matrix $\Sigma$, its condition number $\kappa$, or a range $R$ such that $\|\mu\|_2\leq R$. 
This allows the data analyst to use it without spending privacy budget to estimate these hyperparameters (or guessing them, affecting its accuracy). 
Additionally, not only does the algorithm not need to know $R$, but also its accuracy does not depend on it at all.

Second, it is invariant under invertible affine transformation. If one translates, stretches, and/or rotates the data, runs the algorithm, and then reverses the transformation on the output, the end-to-end algorithm does not change. 
This follows from the fact that Tukey depth itself is affine-invariant.

Third, it has asymptotically optimal accuracy guarantees (Theorem~\ref{thm:BGSUZ}). 
The guarantee holds with respect to Mahalanobis distance $\|\hat{\mu}-\mu\|_\Sigma=\|\Sigma^{-1/2}(\hat{\mu}-\mu)\|_2$, an error metric which tightly captures the uncertainty of the true mean and characterizes the total variation distance between $\cN(\mu, \Sigma)$ and $\cN(\hat{\mu}, \Sigma)$ up to constants. 
Mahalanobis error $\alpha$ implies a Euclidean guarantee of $\|\hat\mu - \mu\|_2 \leq \alpha \sqrt{\|\Sigma\|_2}$. 
\begin{thm}[Theorem 3.2~\cite{brown2021covariance}]\label{thm:BGSUZ}
    For any $\eps,\delta >0$, the Restricted Tukey Depth Mechanism is $(\eps, \delta)$-differentially private.
    There exists an absolute constant $C$ such that, for any $0<\alpha,\beta,\eps < 1$, $0<\delta\le \frac{1}{2}$, mean $\mu$, and positive definite $\Sigma$, if $x\sim \cN(\mu,\Sigma)^{\otimes n}$ and
    \begin{equation}
        n \ge C\left(\frac{d + \log 1/\beta}{\alpha^2} + \frac{d + \log(1/\alpha\eps \beta)}{\alpha\eps} +  \frac{\log 1/\delta}{\eps} \right),
    \end{equation}
    then with probability at least $1-\beta$, $\|\cA(x)-\mu\|_{\Sigma}\leq\alpha$.
\end{thm}
The dependence on $\delta$ in the sample complexity is decoupled from $d$, which implies that we can ask for $\delta$ to be very small, on the order of $e^{-d}$. 
Furthermore, the constant $C$ is not too large, overall making it an algorithm whose accuracy is expected to be good in practice.

Finally, the algorithm is robust to data corruptions in the strong contamination model (Definition~\ref{def:strong-contamination}).
Theorem~\ref{thm:BGSUZ} still holds (up to a change in constant) when the data comes from an $\alpha$-corruption of data from a Gaussian.

For our experiments, we examine the standard Gaussian Mechanism, the practice-oriented CoinPress~\cite{biswas2020coinpress}, and a number of Tukey-based mechanisms: the Tukey Depth Mechanism over the hypercube of~\cite{liu2021robust} (BoxEM) and the Restricted Tukey Depth Mechanism of BGSUZ (REM) across different notions of depth, including the axis-aligned halfspaces used in~\cite{amin2022easy}.

\begin{figure}[h!]
    \centering
    \begin{subfigure}[t]{0.48\columnwidth}
        \centering
        \includegraphics[width=\linewidth]{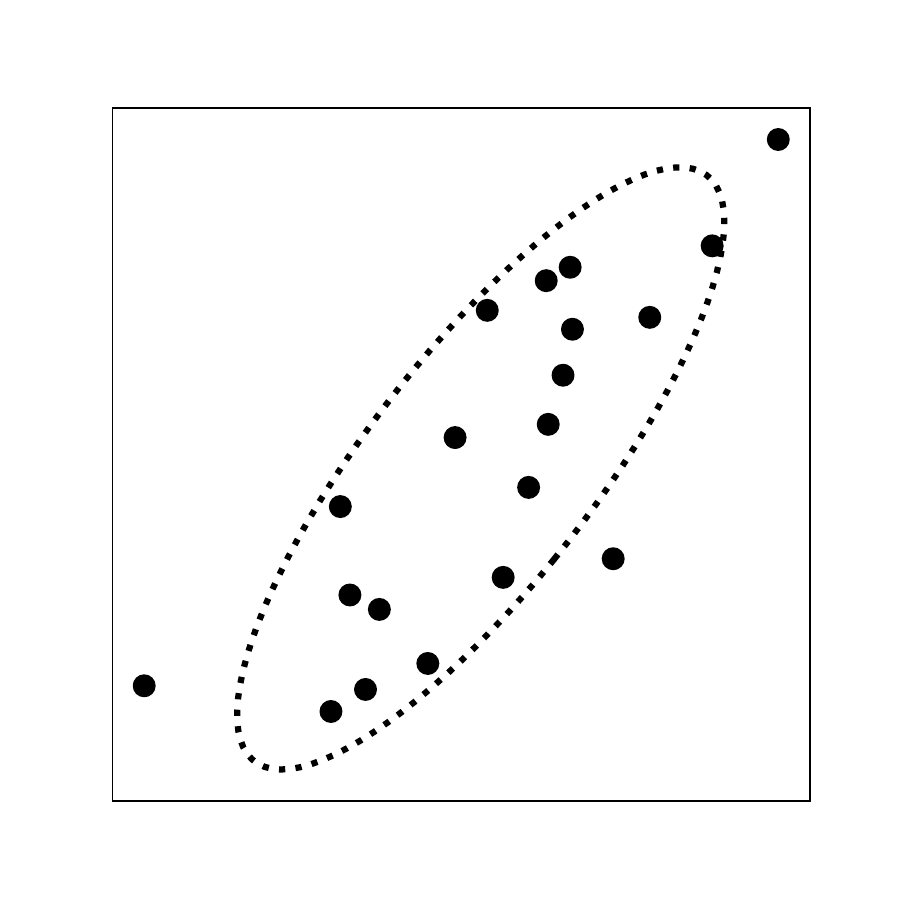} %
        \caption{}
        \label{fig:top-left}
    \end{subfigure}
    \hfill
    \begin{subfigure}[t]{0.48\columnwidth}
        \centering
        \includegraphics[width=\linewidth]{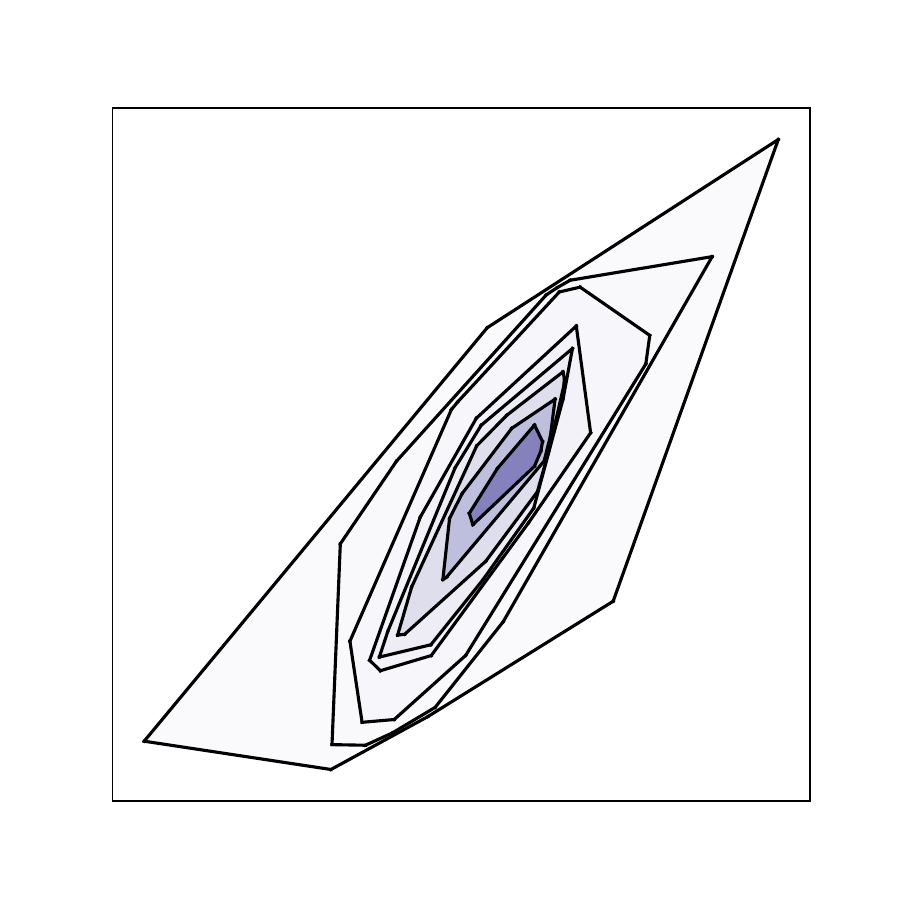} %
        \caption{}
        \label{fig:top-right}
    \end{subfigure}

    \vspace{1em} %
    \begin{subfigure}[t]{0.8\columnwidth}
        \centering
        \includegraphics[width=\linewidth]{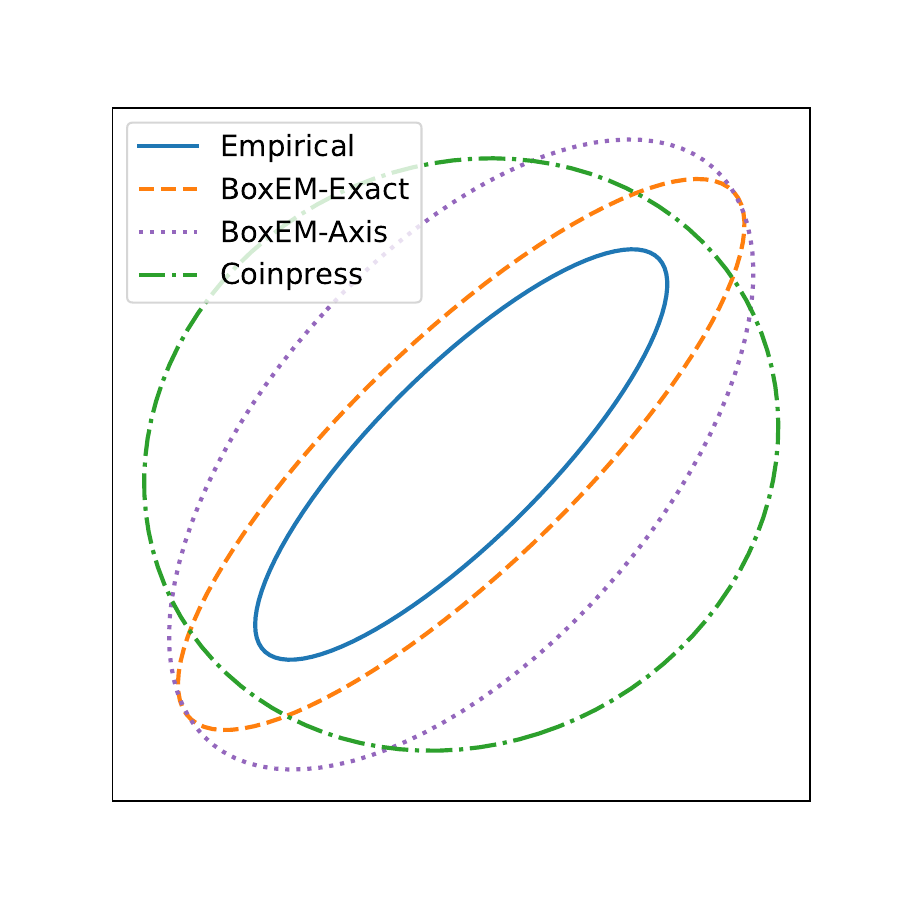} %
        \caption{}
        \label{fig:bottom}
    \end{subfigure}

    \caption{When data arise from a nonspherical distribution (a), the Tukey level sets (b) reflect this. In (c), we show the covariance of output distributions of different mechanisms: the empirical mean and the output of BoxEM-Exact have similar shapes. CoinPress reveals its use of spherical Gaussian noise. 
    (Here we use $n=200$;  CoinPress has relatively high error. To emphasize shape, we scaled the CoinPress covariance  down by a factor of ten.)
    Interestingly, BoxEM with axis-aligned depth seems to sit between the two methods.}
    \label{fig:main-figure}
\end{figure}

\section{Our Implementations}
\label{sec:implementations}
\newcommand{\score}{\tilde{h}}
\newcommand{\floor}[1]{\left\lfloor #1 \right\rfloor}

We implement REM and BoxEM with both exact Tukey depth and random Tukey depth as a score function. In this section, we discuss the details of our implementations, which combine several theoretical and practical tools from the literature on differential privacy, robust statistics, and beyond. 

REM follows the high-level structure outlined in Algorithm~\ref{alg:discrete_tukey_vol_overview}, while BoxEM bypasses the PTR check and samples directly from the exponential mechanism over the box $[-R,R]^d$. 
In Subsection~\ref{sec:high-level-dets}, we start by discussing high-level implementation tools which allow us to instantiate the outline provided Algorithm~\ref{alg:discrete_tukey_vol_overview} in the form of Algorithm~\ref{alg:discrete_tukey_vol}, which is well-specified.
Running Algorithm~\ref{alg:discrete_tukey_vol} on a computer requires constructing Tukey upper-level sets, computing their volumes, and sampling from them uniformly. 
We further discuss the details of our implementation of these steps in Subsection~\ref{sec:low-level-dets}.\footnote{Code will be released upon publication.}

\begin{figure}[!t]
\removelatexerror
\begin{algorithm}[H] \caption{REM with Tukey depth%
}\label{alg:discrete_tukey_vol}
\begin{algorithmic}[1]
\Require{Dataset $x = (x_1,\dots,x_{n})^T \in \mathbb{R}^{n\times d}$.
Privacy parameters: $\eps,\delta>0$. Minimum threshold $t$.}

\State Adjust parameters $\eps_p \gets \frac{\eps}{4}$, $\eps_e \gets \frac{\eps}{2}$, $\delta_p\gets \delta$, $\delta_e\gets \frac{\delta}{e^{2\eps_p}}$. \Comment{Split $\eps$ equally among PTR, $\cM$.}

\Algphase{Volume Computations}
\For{$\ell = \{ 1, \ldots, \floor{n/2}\}$} \label{line:volumes-in}
    \State $V_{\geq \ell} \gets \mathrm{Vol}(\cY_{\geq\ell})$ for $\cY_{\geq\ell}=\{y\in\R^d: T_x(y)\geq \ell\}$. 
\EndFor \label{line:volumes-out}

\Algphase{PTR check}
\State $\mathcal{K} \gets \left\{k \in[0,t): \exists g>0 \text{ s.t. } \frac{V_{\geq t-k-1}}{V_{\geq t+k+g+1}} \cdot e^{-\frac{g\eps_e}{2}} \le \frac{\delta_e}{4e^{\eps_e}} \right\}$
\State $\score(x) \gets \max \left\{ \mathcal{K}, -1\right\}$ 
\If{$\score(x) + \mathrm{Lap}(\frac{1}{\eps_p})< \frac{\log(1/2\delta_p)}{\eps_p}$} 
    \Return \texttt{FAIL}. \label{line:safety-score-check}
\EndIf

\Algphase{Sampling from REM: $\cM$}
\State Draw level $L\in \{t,\ldots,\floor{\frac{n}{2}}\}$ with probability $p_\ell$ as in~\eqref{eq:prob}.\label{line:draw-level}
\State \Return $\tilde\mu \sim \mathrm{Uniform}(\mc{Y}_{\geq L})$.\label{line:uniform-sampling}

\end{algorithmic}
\end{algorithm}
\end{figure}

\subsection{High-Level Implementation Details}\label{sec:high-level-dets}
The key steps in Algorithm~\ref{alg:discrete_tukey_vol_overview} are: 1) sampling from the exponential mechanism with Tukey depth as a score function, and 2) conducting the PTR check. 
\subsubsection{Sampling from the EM with Tukey depth} 
To sample from the (restricted) exponential mechanism with Tukey depth as a score function, we follow the natural approach in~\cite{kaplan2020find}. 
We now sketch the main ideas.

Recall that we aim to sample a point $y$ with probability proportional to $e^{\eps T_x(y)/2}$, where $T_x(y)$ is the Tukey depth of point $y$ with respect to dataset $x$. 
Equivalently, we might pick a random depth $L=\ell$ with probability proportional to $e^{\eps \ell/2}\cdot \mathrm{Vol}(\mc{Y}_{=\ell})$ and then sample uniformly from the set $\mc{Y}_{=L}=\{y\in\R^d: T_x(y)=L\}$. 
With this approach, the probability that we output $W$ that belongs in a Tukey level $\ell\geq t$ is 
\begin{equation}\label{eq:original_prob}
    \Pr[W \in \mc{Y}_{=\ell}] = C \cdot e^{\eps \ell/2}\cdot \mathrm{Vol}(\mc{Y}_{=\ell}),
\end{equation}
for a normalizing constant $C$. 

However, sampling from $\mc{Y}_{=\ell}$ is impractical, since these sets are not convex. 
Tukey upper-level sets $\mc{Y}_{\geq \ell}$, on the other hand, are convex. 
Therefore, we randomly select a depth $L=\ell$ according to a certain probability distribution $\{p_\ell\}$ and then sample uniformly from the convex set $\mc{Y}_{\ge L}$. 

We now show how to construct $\{p_{\ell}\}$ in order to maintain~\eqref{eq:original_prob}.
We write the probability of picking a point at depth exactly $\ell$ as a sum over the depths from $j=t$ to $\ell$, since the upper-level sets at lower depths contain $\mc{Y}_{=\ell}$.
We have
\begin{align*}
    \Pr[W\in \mc{Y}_{=\ell}] & = \sum_{j=t}^\ell \Pr[W\in \mc{Y}_{=\ell} \mid L=j] \Pr[L=j] \\
    &= \sum_{j=t}^m \frac{\mathrm{Vol}(\mc{Y}_{=\ell})}{\mathrm{Vol}(\mc{Y}_{\ge j})} \cdot p_j.
\end{align*}
 Plugging in~\eqref{eq:original_prob} we have
\begin{align*}
\mathrm{Vol}(\mc{Y}_{=\ell})\sum_{j=t}^\ell \frac{p_j}{ \mathrm{Vol}(\mc{Y}_{\ge j})} = C\cdot  \mathrm{Vol}(\mc{Y}_{=\ell})\cdot e^{\eps \ell/2}.
\end{align*}
The volume of $\mc{Y}_{=\ell}$ cancels on both sides. 
Let $x_\ell \defeq p_\ell/\mathrm{Vol}(\mc{Y}_{\ge \ell})$ and $C'=Ce^{\eps t/2}$. 
By definition, we have:
\begin{align*}
    x_t &= C e^{\eps t/2} = C' \\
    x_{t+1} + x_t &= C e^{\eps(t+1)/2} = C' e^{\eps/2}
\end{align*}
which implies $x_{t+1} = C' e^{\eps/2} - C' = C'e^{\eps/2}(1 - e^{-\eps/2})$.
In turn, the latter combined with 
    \begin{align*}
    x_{t+2} + x_{t+1} + x_t = C e^{\eps(t+2)/2}
    \end{align*}
    implies $x_{t+2} = C' e^{2\eps/2}(1 - e^{-\eps/2})$. 
By induction, we retrieve that $x_{\ell} = C e^{\eps \ell/2}(1 - e^{-\eps/2})$, or equivalently, that the probabilities $p_\ell$ should be set as follows
\begin{equation}\label{eq:prob}
    p_\ell = C \cdot \mathrm{Vol}(\mc{Y}_{\ge \ell}) \cdot e^{\eps \ell/2}(1 - e^{-\eps/2}),   
\end{equation}
where $C=\left((1-e^{-\eps/2})\sum_{\ell=t}^{\floor{n/2}} e^{\eps\ell/2} \mathrm{Vol}(\mc{Y}_{\ge \ell})\right)^{-1}$ is the normalizing constant so that $\sum_{\ell=t}^{\floor{n/2}} p_\ell=1$.

Note that this calculation is also useful for BoxEM: the probabilities $p_\ell$ in this case can be retrieved by setting $t=0$ and $p_0 = C \cdot \mathrm{Vol}(\mc{Y}_{\ge 0}) = C (2R)^d$, with the normalizing constant adjusted accordingly.%

Therefore, we conclude that to sample from the REM with Tukey depth, it suffices to compute the volumes of all Tukey upper-level sets $\cY_{\geq m}$, $m\in [t,\floor{n/2}]$ (lines~\ref{line:volumes-in}-\ref{line:volumes-out} in Algorithm~\ref{alg:discrete_tukey_vol}), draw a level $L=\ell$ with probability $p_\ell$ as in~\eqref{eq:prob} (line~\ref{line:draw-level}), and sample uniformly from the convex Tukey region $\cY_{\geq \ell}$ (line~\ref{line:uniform-sampling}). 

\textit{Racing Sampling for Numerical Stability. }
We use racing sampling to implement sampling a depth $L=\ell$ (line~\ref{line:draw-level}), as in \cite{amin2022easy}. 
We want to select depth $L=\ell$ with probability $p_\ell$, proportional to $\mathrm{Vol}(\mc{Y}_{\ge \ell})\cdot e^{\eps\ell/2} \cdot (1-e^{-\eps/2})$ (as in~\eqref{eq:prob}). 
To do this, racing sampling requires drawing independent samples $U_\ell\sim \mathrm{Uniform}[0,1]$ for each $\ell\in[t, \floor{n/2}]$, and computing 
\begin{align}
    Z_\ell = \log\log \frac{1}{U_\ell} - \log \mathrm{Vol}(\mc{Y}_{\ge \ell}) - \frac{\eps \ell}{2} - \log (1-e^{-\frac{\eps}{2}}).
\end{align}
Then $L = \argmin Z_\ell$ will be correctly distributed. See \cite{medina2020duff} for a reference and proof of correctness of racing sampling; the algorithm is attributed to Ilya Mironov. 

\subsubsection{PTR check with Approximate Distance-to-Unsafety}
A key obstacle to implementing REM in Algorithm~\ref{alg:discrete_tukey_vol_overview} is abstract the ``distance to unsafety'' calculation. 
At first glance, it is not clear how one can translate this subroutine to a finite-step algorithm, beyond a na\"ive brute-force search over the space of all datasets. 

Amin, Joseph, Ribero, and Vassilvitskii\cite{amin2022easy} cleverly avoid this search, replacing the exact distance with a carefully constructed approximate distance. 
This approximation is low sensitivity, so it plugs into  the same PTR framework. 
It is a lower bound on the exact distance to unsafety, which implies that (when the private check passes) the restricted exponential mechanism is still private with a rescaling of its parameters. (We discuss the privacy guarantees of our algorithms at the end of this section.) 
Finally, it is easier to compute: it requires only computing the volumes of Tukey upper-level sets on the input dataset. We already need these quantities to execute the exponential mechanism. 
The particular notion of approximate distance comes from a quantity used by BGSUZ in their analysis, while Amin, Joseph, Ribero, and Vassilvitskii~\cite{amin2022easy} show how it can be used algorithmically. 
Subsequent work~\cite{dick2024better} gave an improved distance approximation, which we have not included. 
In our implementation, we use the approximate distance as in~\cite{amin2022easy} to execute the PTR check (line~\ref{line:safety-score-check}). If we already have the volume estimates in hand, the check only takes polynomial time in $n$. 

\subsection{Exact Volume Computation and Sampling}\label{sec:low-level-dets}
With the above tools, we have turned the abstract Algorithm~\ref{alg:discrete_tukey_vol_overview} into Algorithm~\ref{alg:discrete_tukey_vol}, which is much closer to implementation.
Two parts of the computation remain unspecified.
First, how do we compute the volume of the Tukey upper-level sets $\{\mc{Y}_{\le \ell}\}_\ell$?
Second, how do we sample from the uniform distribution over a given $\mc{Y}_{\ge \ell}$?

\subsubsection{Exact Tukey depth} 
Our first implementations of REM and BoxEM use exact Tukey depth. 
In general, the volumes of Tukey upper-level sets can be computed in exponential time via \emph{triangulation}: enumerate the polytope's vertices (i.e., turn the \emph{$H$ representation} into the \emph{$V$ representation}) and partition the body into simplices. 
The volume of each simplex can be efficiently computed and their sum will give us the volume of the polytope, which corresponds to the Tukey upper-level set. 
In our implementation, we rely on the R package \texttt{TukeyRegion} \cite{liu2019fast, fojtik2023exact} and the Python package \texttt{data-depth} to compute the exact Tukey regions and their volumes. 

We sample uniformly from the chosen level set via rejection sampling, drawing proposals uniformly from a bounding box. 
In our experiments, this was not a significant contributor to running time.

\subsubsection{Approximate Tukey depths}
A Tukey region takes $O(n^d)$ to construct, while an exact computation of its volume takes $n^{O(d^2)}$ time. Outside the privacy literature, there is recent work on computationally efficient approximations to the Tukey depth \cite{liu2019fast,fojtik2023exact}. 
Some efficient algorithms compute \emph{exact} depth via heuristics; such algorithms do not suit our purposes, since the manner in which they fail may depend on the data and break privacy.
However, we can use algorithms which \emph{exactly} calculate an approximate Tukey depth that retains its low sensitivity. 

One such approximation is the \emph{axis-aligned} Tukey depth  considered in Amin, Joseph, Ribero, and Vassilvitskii~\cite{amin2022easy}, which replaces the minimum over all halfspaces with a minimum over the $d$ directions of the canonical basis. 
Here, every Tukey region is a high dimensional rectangle, whose volumes are trivial to compute. 
The accuracy guarantee however would be with respect to a hyperrectangle. This provides weaker guarantees for $\ell_2$ error, as dimensionality increases, compared to Tukey depth. Moreover, its accuracy guarantees are no longer affine-invariant. 

Another notion of approximate Tukey depth, which we find more suitable for our purposes, is the \emph{random} Tukey depth \cite{cuesta2008random}, which replaces the minimum over all halfspaces with one over a randomly chosen list of $\directions$ halfspaces. If $\directions\rightarrow n^d$, random Tukey depth would behave as exact Tukey depth. However, $\directions=O(d)$ seem enough directions in practice to maintain good accuracy, while achieving better running time.\footnote{\cite{briend2023qualityrandomizedapproximationstukeys} provide an explanation for this choice: they show that $k=O(d)$ directions suffice, in some parameter regimes, for the random Tukey depth of a point to be close to its exact Tukey depth, for points of high depth. Since our algorithms usually output a point of high depth, $k=O(d)$ is likely enough to yield good practical accuracy.
}

Our second implementation of REM and BoxEM substitutes random Tukey depth for exact. 
We begin the algorithm by choosing a number $\directions$ of halfspaces uniformly at random. Crucially, they are selected independently of the data. 
Each halfspace $j\in [k]$ is determined by a vector $a_j\in \R^d$. 
This means the $\ell$-th level set $\cY_{\ge \ell}$, is now a polytope $P = \{x : Ax\le b_i\}$, where the $j$-th row of $A$ is the vector $a_j$. Thus, the construction of all upper-level sets takes only $nkd$ time. 

There are no known deterministic algorithms for computing polytope volumes whose running time scales only polynomially with the dimension. 
The fastest-known algorithms for exactly computing the volume of a $d$-dimensional polytope with $k$ constraints take $k^d$ time. 
However, a number of practical implementations scale to several dimensions. 
The \texttt{polytope} package in Python provides a number of useful tools for manipulating polytopes.
We use the \texttt{VINCI} software for polytope volumes~\cite{bueler2000exact} and its implementation of Lasserre's algorithm \cite{lasserre1983analytical}.
In low dimensions, we also use the Quickhull algorithm \cite{barber1996quickhull}, available in Python through \texttt{scipy.spatial} \cite{2020SciPy-NMeth}. See~\cite{emiris2018practical} for further comparisons and discussion.

\subsection{Approximate Volume Computation and Sampling} 
A natural next step to further improve the running time of our implementations is to use approximation algorithms for polytope volume and sampling. There are core challenges to making these approaches private, which we address, along with theoretical advances in polytope volume computation in~\Cref{sec:approximate_volume}. 
However, we implement a version of this approach to illustrate its potential. For this purpose, we use the R package \texttt{Volesti} \cite{volesti}, a polished and efficient implementation of several techniques.

\subsection{Privacy Guarantees} 
The privacy analysis of REM with exact volume computations and sampling is almost completely contained in BGSUZ (see Proposition D.1 and Appendix C.2.1). It follows from an analysis of PTR combined with the exponential mechanism, with the exception that the distance-to-unsafety has now been replaced by its approximate lower bound, $\score$. The latter is $2$-sensitive instead of $1$-sensitive, so the total privacy guarantee is $(2\eps_p+\eps_e, \max\{e^{2\eps_p}\delta_e, \delta_p\})$. The parameters of Algorithm~\ref{alg:discrete_tukey_vol} are appropriately scaled to guarantee $(\eps,\delta)$-DP. 

The privacy guarantees of BoxEM with exact volume computations and sampling follow directly from the guarantees of the standard exponential mechanism.

\section{Experiments and Results}
Our work makes Tukey depth mechanisms a viable practical approach for private mean estimation of data low dimensions. With our new implementations at hand, we investigate the practical performance of BoxEM and REM in different scenarios, in terms of their error and use of computational resources. We demonstrate their superior accuracy in these settings compared to commonly used, Gaussian mechanism based approaches, consistent with findings from theoretical literature.

\textit{Experimental setup.} All experiments are performed on synthetic data, generated by selecting a point $\mu$ uniformly at random from the sphere of radius $3$ and then drawing points independently from $\mc{N}(0,\id)$. 
Privacy parameters are set to be $\eps=1$ and $\delta=10^{-6}$ throughout (or simply $\eps=1$ for pure DP mechanisms). 
The range-bounding hyperparameter required by all private mechanisms (except REM) is set to $R=10$, unless noted otherwise. 
The average error is computed over 10 trials, unless noted otherwise. Error bars represent 95\% confidence intervals.\footnote{Formally, on trial $i\in[10]$ we compute error $e^{(i)}=\norm{\mu^{(i)}-\tilde\mu^{(i)}}{2}$ and we plot $\hat{e}\pm 1.96 \frac{\hat\sigma}{\sqrt{100}}$, where $\hat\sigma^2 = \frac{1}{100}\sum_i (e^{(i)}-\hat e)^2$.} 
The \texttt{empirical} curve always denotes the baseline (non-private) empirical error $\norm{\hat\mu-\mu}{2}$, where $\hat\mu$ is the sample mean $\frac 1 n \sum_i x_i$. 
To isolate the error introduced by privacy, for the private mechanisms we report the distance $\norm{\tilde\mu - \hat\mu}{2}$. 
Furthermore, we omit average errors that are larger than 3, as the algorithm that returns ``0'' achieves this error trivially (since $\|\mu\|_2\leq 3$). 
Finally, since REM has a probability of failure that depends on the particular dataset, we do not report the results if REM fails for at least one trial.\footnote{Given the low number of trials, a single failure event likely represents an infeasible failure rate overall.} 
All experiments were conducted on a 2019 MacBook Pro with a 1.4 GHz Quad-Core Intel Core i5. 

\subsection{Running Time}
\begin{table}[t]
\caption{Running time (in seconds) of Tukey depth mechanisms across samples size and dimension. ``-" means the computation took longer than ten minutes. Random depth used $k=30$ directions.}
\begin{center}
\begin{tabular}{c|cccc|ccc|}
\cline{2-8}
\multicolumn{1}{l|}{} & \multicolumn{4}{c|}{Random Depth} & \multicolumn{3}{c|}{Exact Depth} \\ \hline
\multicolumn{1}{|c|}{$n$} 
    & \multicolumn{1}{l|}{$d=2$} & \multicolumn{1}{l|}{$d=3$} & \multicolumn{1}{l|}{$d=4$} & \multicolumn{1}{l|}{$d=5$} 
    & \multicolumn{1}{l|}{$d=2$} & \multicolumn{1}{l|}{$d=3$} & \multicolumn{1}{l|}{$d=4$} 
        \\ \hline\hline
\multicolumn{1}{|c|}{50}   
    & \multicolumn{1}{c|}{0.6}   & \multicolumn{1}{c|}{0.7}   & \multicolumn{1}{c|}{2.0}  & 24.1
    & \multicolumn{1}{c|}{1.0}   & \multicolumn{1}{c|}{0.9}  & 34.6
        \\ \hline
\multicolumn{1}{|c|}{100}  
    &\multicolumn{1}{c|}{1.1}   & \multicolumn{1}{c|}{1.1}   & \multicolumn{1}{c|}{4.0}  & 50.0
    & \multicolumn{1}{c|}{0.4}   & \multicolumn{1}{c|}{14.7}  & -                          
    \\ \hline
\multicolumn{1}{|c|}{500}  
    & \multicolumn{1}{c|}{4.9}   & \multicolumn{1}{c|}{5.2}   &   \multicolumn{1}{c|}{19.8}  & 252.1                    
    & \multicolumn{1}{c|}{6.6}  & \multicolumn{1}{c|}{-}     & -                          
        \\ \hline
\multicolumn{1}{|c|}{1000} 
    & \multicolumn{1}{c|}{10.0}   & \multicolumn{1}{c|}{9.4}  &  \multicolumn{1}{c|}{38.1}  & -                     
    & \multicolumn{1}{c|}{46.0}   & \multicolumn{1}{c|}{-}     & -                          
        \\ \hline
\multicolumn{1}{|c|}{2000} 
    & \multicolumn{1}{c|}{21.9}  & \multicolumn{1}{c|}{19.4}  &  \multicolumn{1}{c|}{75.2}  & -                       
    & \multicolumn{1}{c|}{376.2}     & \multicolumn{1}{c|}{-}     & -                          
        \\ \hline
\end{tabular}
\label{tab:running_times}
\end{center}
\end{table}

First, we establish that our new implementations indeed make Tukey depth mechanisms viable alternatives in low or modest dimensions. 
Table~\ref{tab:running_times} lists running times for BoxEM with exact or random Tukey depth, for up to $n=2000$ samples and $d=5$ dimensions. 
These times are dominated by volume computation and thus representative for both REM and BoxEM. 
As expected, using random Tukey depth with $k=30$ directions  as a score function provides a significant speed-up over exact Tukey depth in all regimes.\footnote{\cite{cuesta2008random} determine that, for all experiment regimes we consider, the convergence of random Tukey depth to the Mahalanobis norm starts becoming negligible after $k=29$ directions.}
This is aligned with theoretical results, as computing the volume of $n$ random Tukey regions takes asymptotically $nk^d$ time, compared to $n^{O(d^2)}$ required for 
exact Tukey regions. 
Overall, Tukey depth mechanisms seem well-suited for deployment in offline scenarios where accuracy comes at a premium and the sample size and dimension are both relatively modest.

\subsection{Better Accuracy in Practice}
Next, we explore the accuracy of Tukey depth mechanisms in practice. We demonstrate that they indeed outperform other commonly used approaches, reinforcing our main claim that pursuing more practical implementations is a valuable endeavor.

In the following experiments, we use our implementations of REM and BoxEM with either exact or random Tukey depth. We compare to two baselines: the folklore Gaussian Mechanism and the practice-oriented CoinPress (code by~\cite{biswas2020coinpress}). The simplest of these, the Gaussian Mechanism, clips the data to an $\ell_2$ ball of radius $R$, and releases the resulting empirical mean with appropriately scaled spherical Gaussian noise. %
CoinPress~\cite{kamath_KLSU19, biswas2020coinpress} iteratively refines an estimate of the mean, which results in its error depending only logarithmically on $R$.

\begin{figure}[t]
\centerline{
\includegraphics[width=0.8\columnwidth]{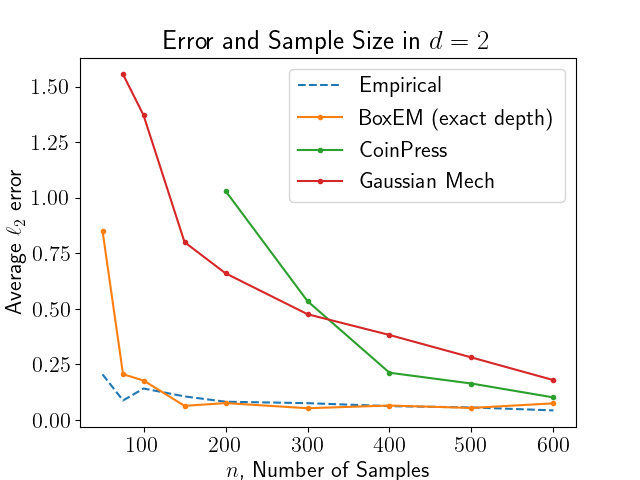}}
\caption{Mechanisms' $\ell_2$ error as a function of sample size.
\texttt{Empirical} represents error due to sampling. Other lines quantify the ``cost of privacy,'' i.e., the difference between the empirical mean and the private estimate. 
The Tukey mechanism introduces error comparable to the empirical error at small samples sizes.}
\label{fig:error_sample_size}
\end{figure}

Fig.~\ref{fig:error_sample_size} shows that the Tukey depth mechanism BoxEM performs much better than the Gaussian noise based mechanisms in this setting. 
In particular, BoxEM introduces error that is comparable to the empirical one, achieving ``privacy for free'' in practice, even at very small sample sizes. Additionally, it does so while satisfying a stronger pure-DP guarantee. 
We note that although in theory the error of both BoxEM and CoinPress is asymptotically the same (and dependent logarithmically on $R$), BoxEM consistently outperforms CoinPress in this setting. This may partially be due to the larger constants involved in Gaussian noise based approaches, compared to the ones based on Tukey depth.

\begin{figure}[t]
\centerline{
\includegraphics[width=0.8\columnwidth]{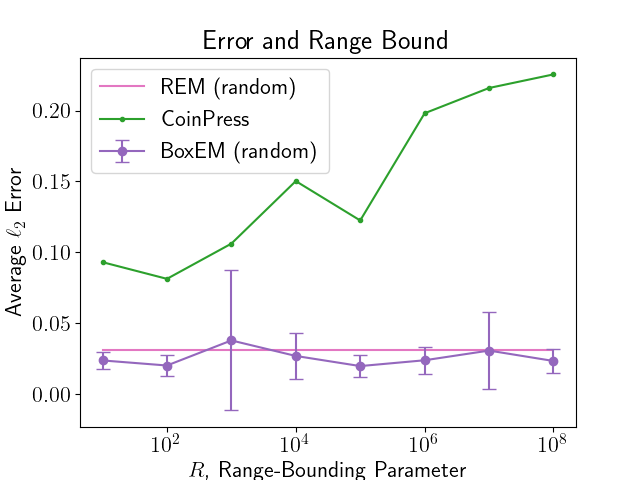}}
\caption{
Mechanisms' $\ell_2$ error as a function of $R$, the range-bounding hyperparameter. 
Tukey depth mechanisms (BoxEM, REM) exhibit essentially no dependence on $R$. %
Note the log-log scale. This experiment uses $n=1000$, $d=2$. Random depth used $k=30$ directions.}
\label{fig:error_range_parameter}
\end{figure}

Fig.~\ref{fig:error_range_parameter} further investigates the effect of the range-bounding parameter $R$ on the error. 
The range-bounding parameter represents prior knowledge about the data; dependence on such a quantity is necessary for algorithms satisfying pure or concentrated differential privacy. However, this quantity can be large: a mild dependence or no dependence is desirable.
In this experiment, we examine how $R$ affects the error of CoinPress, as well as BoxEM and REM with random Tukey depth. 
We do not compare to the Gaussian mechanism since its error scales linearly with $R$ and is clearly only applicable for small values of the hyperparameter.

Fig.~\ref{fig:error_range_parameter} confirms that CoinPress's error suffers form a $\mathrm{log}(R)$ dependence, which is expected as this dependence is not merely a result of the analysis but appears directly in the magnitude of the added noise. 
REM's error (provably) has no dependence on $R$. 
However, BoxEM appears to have \emph{no} dependence on $R$ as well, although its theoretical analysis includes a logarithmic dependence. 
Recall that in BoxEM,
$R$ controls the size of the bounding box, and thus increasing $R$ places more of the exponential mechanism's mass outside the convex hull of the data (i.e., points with Tukey depth zero). 
When the mechanism samples from this space, we expect high error, roughly $\Omega(\sqrt{d}R)$.
However, in the regime of these experiments, with high probability the mechanism samples from a point with non-zero depth.\footnote{Consider $d=2$, $n=1000$, and $\eps=1$, as in this experiment. For $R=10^{10}$, we draw from outside the convex hull with probability proportional to $e^{\eps \cdot 0} \cdot (2R)^2 \approx 10^{26}$.  Empirically, the set of depth $\ge n/4$ has volume around 1, so it receives weight proportional to $e^{\eps n/4}\cdot 1\approx 10^{108}$, dominating the space outside the convex hull.} 
Conditioned on this event, the output of the mechanism is, in fact, independent of $R$. 
Overall, we empirically observe that Tukey depth mechanisms (both REM and BoxEM) are not affected by prior knowledge on the distribution and present more suitable alternatives, particularly when known bounds are loose. 

\subsection{Comparison between Different Notions of Depth}

\begin{figure}[t]
\centerline{
\includegraphics[width=0.8\columnwidth]{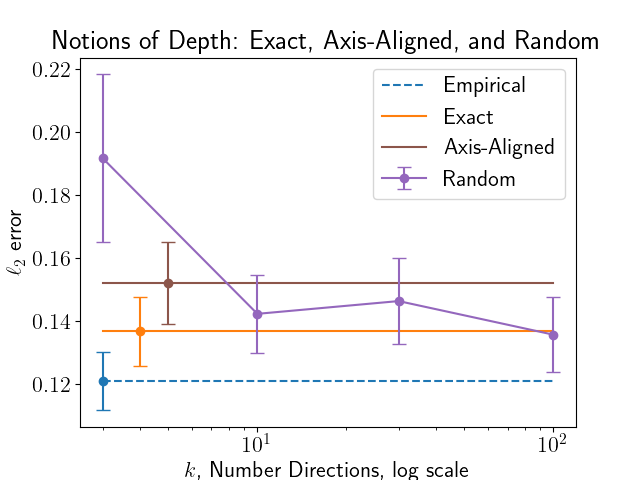}}
\caption{BoxEM's $\ell_2$ error under different notions of depth. 
As more directions $k$ are used, the random depth mechanism soon performs better than axis-aligned depth, and approaches error close to that of exact depth.
This experiment uses $n=200$, $d=2$, and 200 trials. \texttt{Empirical}, \texttt{Exact}, and \texttt{Axis-Aligned} lines each represent a single quantity (which does not depend on $k$). Note the $\log$-scale in $k$.}
\label{fig:error_num_directions}
\end{figure}

The family of Tukey depth mechanisms has good performance in theory and in practice. However, the choice of the depth function plays an important role in the computation-accuracy trade-off these algorithms exhibit: using exact Tukey depth offers tight and affine-invariant accuracy guarantees but slow running time, whereas axis-aligned Tukey depth is very fast to compute but only guarantees a hyperrectangle-shaped confidence interval, suffering larger error in high dimensions. 

Fig.~\ref{fig:error_num_directions} illustrates how various notions of depth affect the error.
The axis-aligned notion of depth used and implemented by \cite{amin2022easy} has higher error than the (much slower) exact-depth algorithm.
The algorithm which computes depth with respect to random directions presents a middle-ground alternative, which allows us to trade-off computation for lower error: as the number of directions increases, the error soon falls under that of axis-aligned depth, and it approaches that of exact depth. We expect the difference between axis-aligned and random Tukey depth to be more stark in higher dimensions. However, making this comparison in high dimensions would be prohibitively slow with our current implementations. 
These results empirically had high variability, so they were averaged over 200 trials. 

\editblock
\subsection{Robustness}
One of the main advantages of Tukey depth mechanisms is that they are robust to data corruptions in the strong contamination model (Definition~\ref{def:strong-contamination}). See Section 3.1 in~\cite{brown2021covariance} for a formal statement of the robustness guarantee of REM and Sections I, J in~\cite{liu2021robust} for BoxEM.
In these experiments, we generate data from a mixture of two Gaussians: one, the ``clean'' data, is $\mathcal{N}(0,\id)$, while the ``corrupted'' data sits at $\mc{N}(s\vec{1}, \frac{1}{10}\id)$, where $s\in \mathbb{R}$ controls the scale of the corruption.  This experiment uses $n=500$, $d=2$, and report averages across 10 trials. 
We also vary the mixture weights between the two components to represent how much of the distribution is corrupted.

In Figure~\ref{fig:robust_rate}, we observe that Tukey depth mechanisms are much less sensitive to the fraction of corruptions compared to CoinPress, as expected. All three private methods appear insensitive to the location of the corrupted data (Figure~\ref{fig:robust_location}). Recall that CoinPress truncates the data in range $R=10$, which offers some level of stability when the location of the corrupted data moves farther from that of the clean data. 

\begin{figure}[t]
\centerline{
\includegraphics[width=0.9\columnwidth]{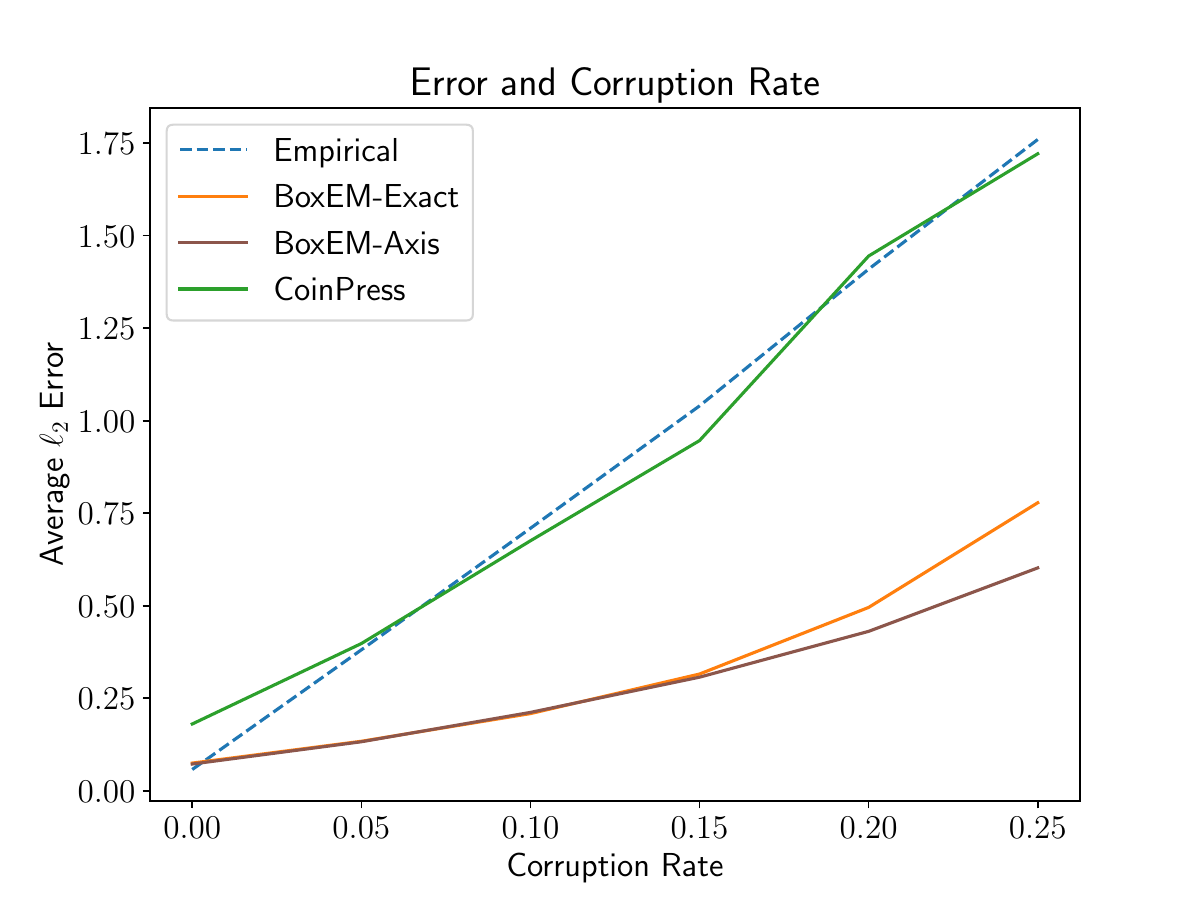}}
\caption{The methods based on the exponential mechanism appear less sensitive to the fraction of outliers.}
\label{fig:robust_rate}
\end{figure}

\begin{figure}[t]
\centerline{
\includegraphics[width=0.9\columnwidth]{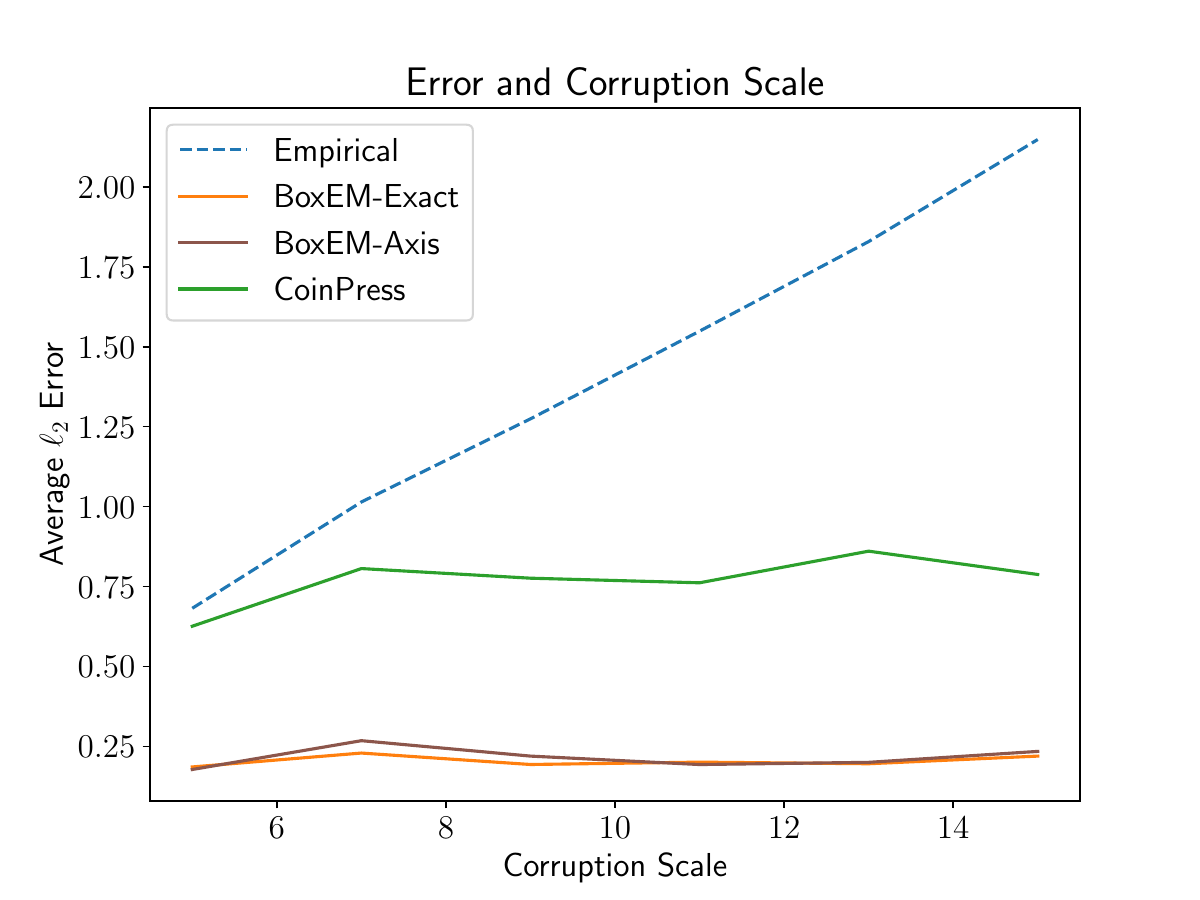}}
\caption{In this experiment, all three methods private appear insensitive to the location of the corrupted data.}
\label{fig:robust_location}
\end{figure}

\editblockdone

\section{To Higher Dimensions Through Sampling}\label{sec:approximate_volume}

Once we have substituted the random depth for exact Tukey depth, our main computational challenges involve sampling from uniform distributions over polytopes and computing their volume.
Above, we used exact solutions (rejection sampling and exact algorithms for the volume) whose running time was exponential in the dimension.
However, there is a long and celebrated line of work providing polynomial-time Markov Chain Monte Carlo algorithms for approximate versions of these tasks.
These algorithms are not purely theoretical: there are many implementations that empirically demonstrate fast mixing times and, subsequently, practical volume approximations.
See \cite{lee2018convergence,gatmiry2023sampling,laddha2020strong} for recent work and discussions.

Formally, we need an efficient algorithm that provides \emph{probably approximately correct (PAC)} volume estimates for polytopes.
\begin{defn}\label{def:volume-oracle-initial}
    Let $A\in \mathbb{R}^{k\times d}$ and $b\in \mathbb{R}^k$ define a polytope $P=\{ x : Ax\le b\}$.
    An \emph{$(\eta,\beta)$-PAC volume oracle} accepts $(A,b)$ and parameters $\eta,\beta\in (0,1)$ and returns a number $\hat{V}$ such that, with probability at least $1-\beta$, 
    \[
        (1-\eta)\mathrm{Vol}(P)\le \hat{V}\le (1+\eta)\mathrm{Vol}(P).
    \]
\end{defn}

The relevant literature provides algorithms that implement this procedure using time polynomial in $d, k, \frac 1 \eta$, and $\log(1/\beta)$.

\subsection{A Call for Provably Practical Volume Estimation}
However, our application uses the PAC guarantees to establish privacy.
Thus, we care not only about the asymptotics but also about constants and logarithmic factors.
To the best of our knowledge, all existing theoretical analyses of polynomial-time polytope volume estimation algorithms incur large constant factors.
For example, methods based on Riemannian Hamiltonian Monte Carlo mix extremely quickly in practice but the theoretical analysis says the underlying Markov chain requires more than $10^{40}\cdot d^{2}$ steps to ensure mixing \cite{kook2023condition}.
Similarly, the hit-and-run technique has multiple practical implementations but seems to require over $10^{10}\cdot  d^3$ steps in order for the theoretical guarantees to apply \cite{lovasz2003hit}.
Even the simplest Markov chains analyzed in this literature, the ball walk and its affine-invariant analog the Dikin walk, have complex proofs that result in  impractical theoretical guarantees for volume estimation \cite{kannan1997random,kannan2009random,laddha2020strong,sachdeva2016mixing}.
It is no surprise that such large constants appear: the analyses are quite  involved and the authors aimed only for tight asymptotic expressions.
However, these constants do not seem to be inherent, as the algorithms appear to mix quite quickly in practice. 
Any theoretical analysis yielding formal guarantees under practical computational limits would immediately yield an improvement in private mean estimation.

\subsection{Approximate Volumes and Privacy}
When we use an algorithm as a PAC volume oracle, the quality of approximation affects our privacy guarantees.
We will now sketch that analysis, the bulk of which appeared in \cite{kaplan2020find}.
We provide necessary additional calculations in Appendix~\ref{app:approximate}. 
Through these claims, we decouple progress on privacy from progress on provably practical volume estimation: any $(\eta,\beta)$-PAC volume oracle with practical running time can be dropped in to immediately yield an improved Tukey depth mechanism.

Informally, to achieve a guarantee of $(\eps,\delta)$-differential privacy  given an $(\eta,\beta)$-PAC guarantee, it suffices to take $\eta$ a bit smaller than $\eps$ and $\beta$ a bit smaller than $\delta$. 
Thus, we can think of the analysis as allocating a constant fraction of the privacy budget to account for the error in the privacy guarantee.
The analysis in \cite{kaplan2020find} shows how to account for this when sampling from the exponential mechanism.
Under the exact exponential mechanism, the probability of selecting a point with Tukey depth $\ell$ is proportional to $V_{=\ell}$, the volume of the points with depth $\ell$.
Under the approximate mechanism, this volume is replaced with our estimate $\hat V_{=\ell}$, which is approximately correct with probability at least $1-\beta$. 
When this holds we have
\[
    \frac{\hat V_{=\ell}}{\sum_{\ell'} \hat V_{=\ell'}}
        \le \frac{(1+\eta) V_{=\ell}}{(1-\eta) \sum_{\ell'} V_{=\ell'}},
\]
a multiplicative increase of $e^{O(\eps)}$ when $\eta\le \eps$ is a small constant. 
The failure probability $\beta$ then folds into the $\delta$ of the privacy guarantee.

The remainder of the privacy analysis, which is new to our work, is to show that the approximate distance-to-unsafety is still low-sensitivity when computed using approximate volumes. 
Since the distance-to-unsafety calculation is based on volume ratios, we can again apply a similar calculation: for any levels $\ell, k$ we have
\[
    \frac{\hat{V}_\ell}{\hat{V}_k}\le \frac{1+\eta}{1-\eta}\cdot \frac{\mathrm{Vol}(\mc{Y}_{\ge\ell})}{\mathrm{Vol}(\mc{Y}_{\ge k})},
\]
a multiplicative increase of at most $e^{O(\eps)}$ whenever $\eta\le \eps$ is a sufficiently small constant.

\subsection{Proof-of-Concept without Privacy Guarantees}
As an initial step in this direction, we conduct an experiment using existing implementations but do not run the Markov chains for the time needed to ensure theoretical guarantees.
Thus, this algorithm is one which might be private (if the Markov chain indeed mixes quickly) but currently is not proved to be so. 

We work with the \texttt{Volesti} package \cite{volesti}, which efficiently implements several random walks and has an established history of development.
Figure~\ref{fig:error_sample_size_d10} shows results in $d=10$ dimensions.
As in Figure~\ref{fig:error_sample_size}, we see that the error due to privacy is comparable to the sampling error when $n \gtrsim 75 d$.

\begin{figure}[t]
\centerline{
\includegraphics[width=0.8\columnwidth]{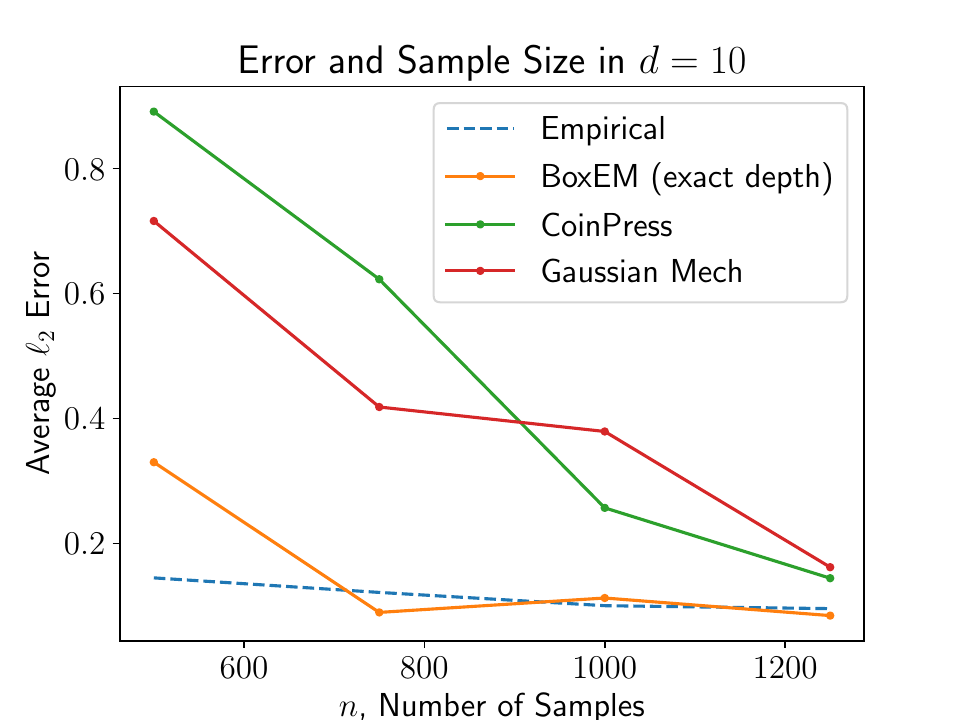}}
\caption{The analog of Fig.~\ref{fig:error_sample_size} in $d=10$ dimensions using \texttt{volesti}.
The Tukey mechanism introduces error comparable to the empirical error at small samples sizes. Results are from one trial.}
\label{fig:error_sample_size_d10}
\end{figure}

Algorithmically, we selected random-directions hit-and-run with Gaussian cooling to approximate the volume, as these are analyzed in the literature.
The distribution of hit-and-run rapidly approaches the uniform distribution over the polytope, mixing in $\tilde{O}(d^3)$ steps of the Markov chain \cite{lovasz1993random,lovasz2003hit}.
Each step of the chain appears to require a logarithmic number evaluations to whether or not a point lies in the polytope, requiring $\tilde{O}(kd)$ time.
Combined with appropriate preprocessing and the annealing technique known as Gaussian cooling, one can obtain an approximation algorithm for the volume requiring, in our case, $\tilde{O}(k d^5)$ time \cite{lovasz2006simulated}.
However, as we discuss above, the constants involved imply we cannot currently apply these theorems to practical algorithms.

Instead, we select hyperparameters heuristically to find a reasonable tradeoff between empirical performance and running time.
In our experiments, we settled on roughly 20 seconds per volume computation, yielding an overall running time of $10 n$ seconds, since there are at most $n/2$ level sets.
As this code is preliminary, we did not work to optimize the running time.
A careful implementation would take advantage of the specifics of the task.
We need approximate volumes for a family of polytopes $\{P_i\}_{i}$, so the computation could easily be parallelized.
Moreover, for all $i$ we have  $P_i = \{x: Ax\le b_i\}$ for some fixed $A$ and $P_{i+1}\subseteq P_i$.
In modest dimensions, the volumes themselves also change slowly, suggesting that $\mathrm{Vol}(P_{i+1})$ could be quickly approximated relative to a previously computed $\mathrm{Vol}(P_i)$ using only a few samples from $P_i$.

\section{Conclusions and Practical Guidance}

In this work, we provide a new implementation of differentially private estimator for multivariate means.
The algorithm has strong theoretical properties: it is accurate at small sample sizes, robust to adversarial corruptions, and invariant to affine transformations of the data.
Unfortunately, the theoretically analyzed version of the algorithm has large computational demands; it is not clear that it could be made practical in low dimensions.

Building on recent developments in differential privacy, sampling, and robust statistics, we work to overcome the computational barriers.
We first implement an exact version of the Restricted Exponential Mechanism over Tukey depth that is practical in low dimensions.
Next, we consider approximate versions of Tukey depth based on random and axis-aligned directions.
Finally, we explore the use of Monte Carlo polytope volume estimation algorithms in higher dimensions.
We provide the technical tools to use such algorithms as a drop-in subroutine within our implementations and highlight a need for improved theoretical analysis of these (already empirically practical) approximation algorithms.

\editblock
\textbf{Selecting an Algorithm:}
For private mean estimation on small datasets without strong prior knowledge about the location or scale of the data, Tukey depth mechanisms may be the only feasible approach. REM does not require such prior knowledge; the error of BoxEM depends only weakly on the strength of our prior, and in some regimes not at all. 

Given more samples or strong prior knowledge, Gaussian-based mechanisms are a viable solution. However, Tukey depth mechanisms additionally offer robustness and affine invariance: these are often important requirements (for example, guaranteeing accuracy even when our assumptions about the data distribution are slightly wrong).

Among Tukey depth mechanisms, REM and BoxEM with exact or random Tukey depth (as long as the number of directions is $k=\Omega(d)$) provide accuracy in Mahalanobis norm. Even for estimation in Euclidean norm, axis-aligned Tukey depth incurs $\sqrt{d}$ times more error than the two alternatives. This is acceptable in low dimensions, but becomes suboptimal as the dimension grows. 

Overall, we contend that this collection of Tukey depth mechanisms, both BoxEM and REM across exact, random, and axis-aligned depth, represent a promising part of the practical toolkit for differentially private mean estimation.  
\editblockdone

\section*{Acknowledgments}

GB is supported by Microsoft Grant for Customer Experience Innovation and the National Science Foundation under grant numbers 2019844, 2112471, and 2229876.
LZ is supported by a Foundations of Data Science Institute and Simons Institute for the Theory of Computing postdoctoral fellowship, and wishes to acknowledge funding from the European Union (ERC-2022-SYG-OCEAN-101071601). 
Views and opinions expressed are however those of the author(s) only and do not necessarily reflect those of the European Union or the European Research Council Executive Agency. Neither the European Union nor the granting authority can be held responsible for them.

\bibliographystyle{IEEEtran}

\bibliography{bibliography}

\appendices

\begin{table}[t]
\caption{Collected asymptotic running times for various operations, omitting $\log$ factors. $n$: sample size, $d$: dimension, $k$: number of directions with respect to which we compute random Tukey depth, $R$: radius of enclosing ball when convex body is in isotropic position.}
\centering
\begin{tabular}{|l|c|c|}
\hline
\textbf{Task}                        & \textbf{Time} & \textbf{References}                         \\ \hline
Exact Tukey Depth $T_x(y)$           &  $n^{d-1}$     &         \cite{dyckerhoff2016exact}                                    \\ \hline
Construct One Tukey Polytope         &   $n^d$   & \cite{liu2019fast}    \\ \hline
Construct All Tukey Polytopes        &  $n^d$         & \cite{fojtik2023exact}     \\ \hline
Exact Volume of Polytope             &  $k^d$         & \cite{chazelle1993optimal}                                       \\ \hline
Construct All Random Tukey Polytopes        & $nkd$         &                                             \\ \hline
Approx. Polytope Volume (Ball Walk)  &    $d^4R^2$            &     \cite{lovasz1993random}                                      \\ \hline
Approx. Polytope Volume (Best-Known) &        $k^{4/3}d^{10/3}$        & \cite{gatmiry2023sampling} \\ \hline

\update{Approx. Polytope Volume (Hit-and-Run)} & $kd^5$ & \cite{lovasz2003hit, lovasz1993random, lovasz2006simulated} \\ \hline
\end{tabular}
\label{tab:asymptotic_times}
\end{table}

\section{Implementations via Approximate Volumes and Sampling}\label{app:approximate}
As in \cite{kaplan2020find}, Tukey depth mechanisms can work even when we have only approximations to the volumes and we sample almost-uniformly from the convex sets.
Of course, we have to track the effects of the approximations. As discussed in~\Cref{sec:approximate_volume}, finding appropriate and practical approximation algorithms to plug into this approach is challenging.  In this section, we present an analysis which, given formal approximation guarantees of a volume oracle and uniform sampler, retrieves the resulting privacy guarantees of our algorithm. 

We note that the calculations in this section apply to any of the notions of depth we consider, including the exact and random Tukey depths.

\subsection{Computational Details} 
For the volume approximation, we assume access to a volume oracle (Def.~\ref{def:volume-oracle}) which replaces the exact volume computations of lines~\ref{line:volumes-in}-\ref{line:volumes-out} in Algorithm~\ref{alg:discrete_tukey_vol}.

The following definition extends the PAC volume oracle of Definition~\ref{def:volume-oracle-initial} to admit multiple polytopes at once, as it may be possible to share computation across polytopes.
\begin{defn}\label{def:volume-oracle}
    An \emph{$(\eta,\beta)$-PAC volume oracle} $\cV_{\eta,\beta}$ accepts a list of polytopes $\{P_i\}_{i\in \mc{I}}$, each presented in H-representation as a matrix-vector pair $(A^{(i)},b^{(i)})$, and parameters $\eta,\beta\in (0,1)$ and returns a list $\{\hat{V}_i\}_{i\in\mc{I}}$ such that, with probability at least $1-\beta$, for all $i\in \mc{I}$ we have
    \[
        (1-\eta)\mathrm{Vol}(P_i)\le \hat V_i \le (1+\eta)\mathrm{Vol}(P_i).
    \]
\end{defn}
One would then apply this to the polytopes corresponding to the (random) Tukey upper-level sets.

In the exact case, our algorithm draws $L=\ell$ in~\cref{line:draw-level} with probability $p_\ell$ as computed in~\cref{eq:prob} above. Similarly, we will now use these approximate volumes to choose $L=\ell$ in~\cref{line:draw-level} with probability
\begin{equation}\label{eq:approx-prob}
    p_\ell' = C' \cdot V'_{\ge \ell} e^{\eps_e \ell/2}(1- e^{-\eps_e/2}),
\end{equation}
where $C'=\left((1-e^{-\eps_e/2})\sum_{\ell=t}^{\floor{n/2}} e^{\eps_e\ell/2} V'_{\ge \ell}\right)^{-1}$.

After choosing level $L$, the algorithm will use an approximate uniform sampler to sample a point from $\mc{Y}_{\geq L}$. 

\begin{defn}\label{def:approx-unif-sampler}
    Let $\mc{Y}$ be a convex body in $\R^d$ and $\tau,\zeta\in(0,1)$. An {\em approximate uniform sampler} $\cS_{\tau,\zeta}$ is an algorithm which, with probability $1-\zeta$, returns a random point from $\mc{Y}$ with pdf $U'$ whose total variation distance from  $\mathrm{Uniform}(\mc{Y})$ is at most $\tau$.
\end{defn}

With these primitives in hand, an implementation of REM would follow as presented in Algorithm~\ref{alg:discrete_tukey_vol_approx}.

\begin{figure}[!t]
\removelatexerror
\begin{algorithm}[H] \caption{REM  %
via Approximate Volumes and Sampling}\label{alg:discrete_tukey_vol_approx}
\begin{algorithmic}[1]
\Require{Dataset $x = (x_1,\dots,x_{n})^T \in \mathbb{R}^{n\times d}$.
Privacy parameters: $\eps,\delta>0$. Minimum threshold $t$.}

\State Adjust parameters $\eps_p \gets \frac{\eps}{6}$, $\eps_e\gets \frac{2\eps}{5}$, $\delta_p\gets \frac{\delta}{2}$, $\delta_e\gets \frac{\delta}{4e^{11\eps/20}}$. \Comment{Split $\eps$ equally among PTR, $\cM'$.}
\State Set volume approximation parameters $\eta \gets \frac{e^{\eps/20}-1}{e^{\eps/20}+1}$, $\beta\gets \min\{\frac{1}{4},\frac{\delta}{8(2e^\eps+1)}\}$.
\State Set approximate sampling parameters $\zeta\gets \min\{\frac{1}{4}, \frac{\delta}{8(2e^\eps+1)}\}$, $\tau\gets \frac{\delta}{4e^{11\eps/20}(1+e^{\eps/2})}$.

\Algphase{Volume Computations}
\update{   \State $\{V'_{\geq \ell}\}_{\ell\in[\floor{n/2}]} \gets \cV_{\eta,\beta}\left(\{\cY_{\ge \ell}\}_{\ell=t}^{\floor{n/2}}\right)$ \label{line:volumes-approx} \Comment{Approximately compute the volumes of polytopes $\{\cY_{\ge \ell}\}_{\ell=t}^{\floor{n/2}}$ corresponding to (random) Tukey upper-level sets.} }

\Algphase{PTR check}
\State $\mathcal{K} \gets \left\{k \in[0,t): \exists g>0 \text{ s.t. } \frac{V'_{\geq t-k-1}}{V'_{\geq t+k+g+1}} \cdot e^{-\frac{g\eps_e}{2}} \le \frac{\delta_e}{4e^{\eps_e}} \right\}$
\State $\score(x) \gets \max \left\{ \mathcal{K}, -1\right\}$ 
\If{$\score(x) + \mathrm{Lap}(\frac{1}{\eps_p})< \frac{\log(1/2\delta_p)}{\eps_p}$} 
    \Return \texttt{FAIL}. \label{line:safety-score-check-approx}
\EndIf

\Algphase{Sampling from REM: $\mc{M}'$}
\State Draw level $L\in \{t,\ldots,\floor{n/2}\}$ with $\Pr[L=\ell]=p_\ell'$ as in~\eqref{eq:approx-prob}. 
\label{line:draw-level-approx}
\State \Return $\tilde\mu' \gets \cS_{\tau,\zeta}(\mc{Y}_{\geq L})$
\label{line:uniform-sampling-approx}

\end{algorithmic}
\end{algorithm}
\end{figure}

\subsection{Privacy Guarantees} 
We can show the following equivalent of Lemma 5.7 from~\cite{kaplan2020find} for the REM.

\begin{lemma}\label{lem:REM-privacy-approx}
Let dataset $x$. Let $\mc{M}$ be the REM of lines~\ref{line:draw-level}-\ref{line:uniform-sampling} in Algorithm~\ref{alg:discrete_tukey_vol} and $\mc{M}'$ the REM of lines~\ref{line:draw-level-approx}-\ref{line:uniform-sampling-approx} in Algorithm~\ref{alg:discrete_tukey_vol_approx} which uses approximate volume oracle $\cV_{\eta,\beta}$ and uniform sampler $\cS_{\tau,\zeta}$. 
Let $G$ be the event that the guarantees of Def.~\ref{def:volume-oracle},~\ref{def:approx-unif-sampler} hold and $B\subseteq\R^d$ be any measurable set. 
Then
\begin{align*}
& \frac{1-\eta}{1+\eta} \Pr[\mc{M}(x)\in B] -\tau\\ 
& \leq \Pr[\mc{M}'(x)\in B \mid G] \\
& \leq \frac{1+\eta}{1-\eta}\Pr[\mc{M}(x)\in B] + \tau.
\end{align*}
\end{lemma}
\begin{IEEEproof}[Proof Sketch]
To prove the lemma, as in~\cite{kaplan2020find}, we condition on G and expand $\Pr[\mc{M}'(x)\in B]= \sum_{\ell=t}^{\floor{n/2}} \Pr[\mc{M}'(x)\in B \mid \mc{M}'(x)\in \mc{Y}_{\ge \ell}] \cdot \Pr[\mc{M}'(x)\in \mc{Y}_{\ge \ell}] = \sum_{\ell=t}^{\floor{n/2}} (\frac{\mathrm{Vol}(B)}{\mathrm{Vol}(\mc{Y}_{\ge \ell})} \pm \tau) \cdot p_{\ell}'$, where the latter holds by the guarantees of the approximate uniform sampler. 
By the guarantees of the PAC volume oracle, we also get that
\begin{equation}\label{eq:approx-prob-guarantee}
\frac{1-\eta}{1+\eta} p_\ell \leq p_\ell' \leq \frac{1+\eta}{1-\eta}p_\ell,
\end{equation}
which would complete the proof.
\end{IEEEproof}

Lemma~\ref{lem:REM-privacy-approx} allows us to connect the privacy parameters of $\cM'$ to the ones of $\cM$: if $G_x$ and $G_y$ are the events that the guarantees of Def.~\ref{def:volume-oracle} and~\ref{def:approx-unif-sampler} hold for the executions of Algorithm~\ref{alg:discrete_tukey_vol_approx} on neighboring datasets $x$ and $y$, respectively, and $G=G_x\land G_y$, then \begin{equation*}
    \cM(x)\approx_{(\eps_e,\delta_e)}\cM(y) \Rightarrow
  \cM'(x)_{\mid G}\approx_{(\eps_2,\delta_2)}\cM'(y)_{\mid G}
\end{equation*} for
\begin{equation}\label{eq:relative-privacy}
  \eps_2=\eps_e+2\log\frac{1+\eta}{1-\eta}, ~~\delta_2=\frac{1+\eta}{1-\eta}\delta_e+\left(1+e^{\eps_e+2\log\frac{1+\eta}{1-\eta}}\right)\tau.
\end{equation}

Next, it remains to track the effect of the approximations on the distance-to-unsafety check. 

\begin{lemma}\label{lem:score-sensitivity-approx} 
Let $G_x$ and $G_y$ be the events that the guarantees of Def.~\ref{def:volume-oracle} and~\ref{def:approx-unif-sampler} hold for the executions of Algorithm~\ref{alg:discrete_tukey_vol_approx} on neighboring datasets $x$ and $y$, respectively, and condition on those events. 
The sensitivity of the approximate distance function $\score'$ (\cref{line:safety-score-check-approx} in Algorithm~\ref{alg:discrete_tukey_vol_approx}) is $$|\score'(x)-\score'(y)|\leq 2\left(\frac{4}{\eps_e}\log\frac{1+\eta}{1-\eta}+1\right).$$ 
The PTR check then satisfies $(8\frac{\eps_p}{\eps_e}\log\frac{1+\eta}{1-\eta}+2\eps_p)$-differential privacy. 
Moreover, $D_H(x,\mathtt{UNSAFE}_{\eps_e,\delta_e,t}) > \score'(x)$.
\end{lemma}
\begin{IEEEproof}[Proof Sketch]
The proof follows Lemma 3.6 of ~\cite{amin2022easy}. 
We can show that for neighboring datasets $x,y$ (in the add/remove model), and fixed $k_x>0, g_x>0$, $$\frac{V'_{\geq t-k_y-1, y}}{V'_{\geq t+k_y+g_y+1, x}} \cdot e^{-g_y\eps_e /2} \leq \frac{V'_{\geq t-k_x-1, x}}{V'_{\geq t+k_x+g_x+1, x}} \cdot e^{-g_x\eps_e /2},$$ for $g_y=g_x+\frac{4}{\eps_e}\log\frac{1+\eta}{1-\eta}>0$ and $k_y=k_x-\frac{4\log\frac{1+\eta}{1-\eta}}{\eps_e}-1$. 
So $\score'(y)\geq \score'(x)-\frac{4}{\eps_e}\log\frac{1+\eta}{1-\eta}-1$. 
The remaining steps of the proof are the same (and we account for a factor of $2$ due to switching to the swap model).

Moreover, we can relate the approximate distance using exact volumes $\score(x)$ (\cref{line:safety-score-check} in Algorithm~\ref{alg:discrete_tukey_vol}) to $\score'(x)$. 
Specifically, $\score'(x)=k$ means that there exists some $g>0$ so that $\frac{V'_{\geq t-k-1, x}}{V'_{\geq t+k+g+1, x}} \cdot e^{-g\eps_e /2}\leq \frac{1-\eta}{1+\eta}\delta_e(4e^{\eps_e})^{-1}$. 
Under $G_x$, this implies that $\frac{V_{\geq t-k-1, x}}{V_{\geq t+k+g+1, x}} \cdot e^{-g\eps_e /2}\leq \frac{1+\eta}{1-\eta}\frac{1-\eta}{1+\eta}\delta_e(4e^{\eps_e})^{-1}=\delta_e(4e^{\eps_e})^{-1}$ so $\score(x)\geq k=\score'(x)$. 
By Lemma 3.8 in~\cite{brown2021covariance}, $D_H(x,\mathtt{UNSAFE}_{\eps_e,\delta_e,t}) > \score(x)$, completing the proof.
\end{IEEEproof}

We derive the overall privacy guarantee of Algorithm~\ref{alg:discrete_tukey_vol_approx}: we first condition on the event $G$ that the approximate volume computations and sampling are successful and then account for this event, using the following general proposition:

\begin{proposition}\label{prop:cond-privacy}
If $\cA(x)_{|G}\approx_{(\eps,\delta)}\cA(y)_{|G}$, then it holds that $\cA(x)\approx_{(\eps,\delta')}\cA(y)$ for $\delta'=\delta+\Pr[\bar{G}]\left(e^\eps\Pr[G]^{-1}+1\right)$.
\end{proposition}
\begin{IEEEproof}
Let $x,y$ be neighboring datasets and $B$ be any measurable set.
\begin{align*}
    & \Pr[A(x)\in B] \\
    & \leq \Pr[A(x)\in B \mid G] + \Pr[\bar{G}] \\
    & \leq e^\eps\Pr[A(y)\in B \mid G] + \delta + \Pr[\bar{G}] \tag{by assumption}\\
    & \leq e^\eps\Pr[A(y)\in B]\Pr[G]^{-1}+\delta +\Pr[\bar{G}] \\
    & = e^\eps\left(\Pr[A(y)\in B] + \Pr[A(y)\in B] (\Pr[G]^{-1}-1)\right)\\
    & \hspace{0.1\columnwidth} +\delta+\Pr[\bar{G}] \\
    & \leq e^\eps\Pr[A(y)\in B]+e^\eps(\Pr[G]^{-1}-1)+\delta+\Pr[\bar{G}]\\
    &= e^\eps\Pr[A(y)\in B]+e^\eps\Pr[\bar{G}]\Pr[G]^{-1}+\delta+\Pr[\bar{G}] 
\end{align*}
This completes the proof of the proposition. 
\end{IEEEproof}

\begin{thm}\label{thm:cond-privacy-approx}
Algorithm~\ref{alg:discrete_tukey_vol_approx} is $(\eps,\delta)$-DP for $$\eps=\left(8\frac{\eps_p}{\eps_e}\log\frac{1+\eta}{1-\eta}+2\eps_p\right)
+\left(\eps_e+2\log\frac{1+\eta}{1-\eta}\right)$$
and \begin{dmath*}
\delta =2(2e^{\eps}+1)(\beta+\zeta) + \max\Big\{\delta_p, e^{8\frac{\eps_p}{\eps_e}\log\frac{1+\eta}{1-\eta}+2\eps_p}\left(\frac{1+\eta}{1-\eta}\delta_e+\left(1+e^{\eps_e+2\log\frac{1+\eta}{1-\eta}}\right)\tau\right)\Big\}.
\end{dmath*}
\end{thm}
\begin{IEEEproof}[Proof Sketch]
Denote the algorithm by $\cA'$ for short. Let $G$ the event that the guarantees of Def.~\ref{def:volume-oracle} and~\ref{def:approx-unif-sampler} hold for the executions of $\cA'$, on neighboring datasets $x$ and $y$ simultaneously.

We follow the proof of Proposition D.1 from~\cite{brown2021covariance} and split the proof into two cases: $\cM(x)\not\approx_{(\eps_e,\delta_e)} \cM(y)$ and $\cM(x)\approx_{(\eps_e,\delta_e)} \cM(y)$. 
The former case is easy: it implies that $D_H(x,\mathtt{UNSAFE}_{\eps_e,\delta_e,t})=D_H(y,\mathtt{UNSAFE}_{\eps_e,\delta_e,t})=0$ so $\score'(x)=\score'(y)=-1$ by Lemma~\ref{lem:score-sensitivity-approx}. 
The PTR check will then guarantee $(0,\delta_p)$-DP in that case since the probability of failure is the same.

Now, assume $\cM(x)\approx_{(\eps_e,\delta_e)} \cM(y)$. 
By Lemma~\ref{lem:score-sensitivity-approx} the probability of failing under $x,y$ can differ by at most a factor of $e^{\eps_1}$ where $\eps_1=8\frac{\eps_p}{\eps_e}\log\frac{1+\eta}{1-\eta}+2\eps_p$.
For a measurable set $B$, and denoting $F=\mathtt{FAIL}$, we break down the probability:
\begin{align*}
& \Pr[\cA'(x)\in B \mid G] \\
&= \Pr[\cA'(x)\in B\cap F\mid G]+ \Pr[\cA'(x)\in B\setminus F \mid G]\\
&= \Pr[\cA'(x)\in B\mid \cA'(x) \in F \land G]\Pr[\cA'(x) \in F \mid G] \\
&\quad + \Pr[\cA'(x)\in B\mid  \cA(x) \notin F \land G]\Pr[\cA'(x)\notin F \mid G] \\
&\le e^{\eps_1}\biggl(\Pr[\cA'(x)\in B\mid \cA(x) \in F \land G]\Pr[\cA'(y) \in F \mid G] \\
&\quad+ \Pr[\cA'(x)\in B\mid  \cA'(x) \notin F \land G]\Pr[\cA'(y)\notin F \mid G] \biggr).
\end{align*}

Since $B$ either contains $\mathtt{FAIL}$ or it doesn't, $$\Pr[\cA'(x)\in B\mid \cA'(x) \in F \land G]=\Pr[\cA'(y)\in B\mid \cA(y) \in F \land G].$$
Furthermore, not failing means we run $\cM'(x)$. Let $\eps_2=\eps_e+2\log\frac{1+\eta}{1-\eta}$ and $\delta_2=\frac{1+\eta}{1-\eta}\delta_e+\left(1+e^{\eps_e+2\log\frac{1+\eta}{1-\eta}}\right)\tau$. By our assumption and~\cref{eq:relative-privacy}, we have
\begin{align*}
& \Pr[\cA'(x)\in B \mid G] \\
&\le e^{\eps_1}\Big(\Pr[\cA'(y)\in B\cap F \mid G]\\ 
& \hspace{0.1\columnwidth} + \Pr[\cM'(x)\in B \mid G]\Pr[\cA'(y)\notin F \mid G]\Big) \\
&\le e^{\eps_1}\Big(\Pr[\cA'(y)\in B\cap F \mid G]\\ 
& \hspace{0.1\columnwidth} + \left(e^{\eps_2} \Pr[\cM'(y)\in B \mid G]+\delta_2\right)\Pr[\cA'(y)\notin F \mid G]\Big).
\end{align*}
To finish the proof, we simplify:
$$\Pr[\cA'(x)\in B \mid G] \le e^{\eps_1+\eps_2}\Pr[\cA'(y)\in B \mid G] + e^{\eps_1}\delta_2.$$

The overall privacy guarantee for both cases, conditioned on $G$, is then $(\eps_G,\delta_G)$-DP for
$$\eps_G=\left(8\frac{\eps_p}{\eps_e}\log\frac{1+\eta}{1-\eta}+2\eps_p\right)
+\left(\eps_e+2\log\frac{1+\eta}{1-\eta}\right)$$
and
\begin{dmath*}
\delta_G =\max\Big\{\delta_p, e^{8\frac{\eps_p}{\eps_e}\log\frac{1+\eta}{1-\eta}+2\eps_p}\left(\frac{1+\eta}{1-\eta}\delta_e+\left(1+e^{\eps_e+2\log\frac{1+\eta}{1-\eta}}\right)\tau\right)\Big\}.
\end{dmath*}

By Proposition~\ref{prop:cond-privacy}, we take into account the event $G$, which has $\Pr[\bar{G}]\leq 2(\beta+\zeta)<1/2$. We conclude that the overall privacy parameters are 
$$\eps=\eps_G%
\text{ and }
\delta\leq \delta_G+\Pr[\bar{G}](\Pr[G]^{-1}e^{\eps}+1)
.$$
\end{IEEEproof}

\editblock
\section{Univariate Gaussian Mean Estimation}\label{sec:univariate_experiments}

We compare two standard approaches for private mean estimation on univariate Gaussians.
The theoretical analysis is straightforward and well-known, but we are not aware of this specific demonstration.
For an in-depth analysis of a similar phenomenon, see \cite{sarathy2022analyzing}.

In one dimension, BoxEM coincides with the folklore exponential mechanism over quantiles, which assigns higher score to values closer to the median.
We compare this with the standard Gaussian mechanism, which clips its inputs and adds appropriately scaled noise.
In these experiments, we fix the data distribution to be $\mathcal{N}(0,1)$. We run the Gaussian mechanism with $(\eps=1, \delta=10^{-6})$-DP and clipping to $[-5,+5]$.
We run BoxEM with $(\eps=1,\delta=0$)-DP over the range $[-5,+5]$.
Results are averaged over 10,000 trials.

We view this setting as favorable to the Gaussian mechanism, since clipping to $\pm 5$ only makes sense with strong prior knowledge that the true mean lies near zero.
Yet~\Cref{fig:univariate_small_sample} demonstrates that BoxEM is much better at small sample sizes; with 100 samples, BoxEM has roughly 1.5 times the error of the empirical mean, while the Gaussian mechanism has roughly three times the error.
This mirrors the multivariate results from~\Cref{fig:error_sample_size}, and the theoretical analysis of these algorithms supports it.
Formally, let $R$ denote the clipping/bounding value ($R=5$ in these experiments).
The exponential mechanism over quantiles achieves, with high probability,
\[
    \lvert\mu-\hat\mu\rvert \lesssim \underbrace{\frac{1}{\sqrt{n}}}_{\text{sampling error}} + \underbrace{\frac{\log R}{\eps n}}_{\text{privacy error}}.
\]
The Gaussian mechanism, on the other hand, achieves
\[
    \lvert\mu-\hat\mu\rvert \lesssim \frac{1}{\sqrt{n}} + \frac{R \sqrt{\log (1/\delta)}}{\eps n}.
\]
With few samples, the privacy noise will dominate.
As these formulas suggest, and as the simulations bear out, the exponential mechanism over quantiles tends to introduce lower noise from privacy in this regime.

As we see in~\Cref{fig:univariate_large_sample}, the story changes as $n$ grows. 
For both mechanisms, the error from privacy dies off with $\frac 1 n$, faster than the sampling error.
Thus the RMSE of the Gaussian mechanism RMSE approaches that of the empirical mean: $\frac{1}{\sqrt{n}}$ exactly.
The exponential mechanism over quantiles, however, has error that approaches that of the \emph{median}.
For sufficiently large $n$, the RMSE of the median approaches $\sqrt{\frac{\pi}{2 n}}$, worse by a constant factor than the empirical mean. 
Thus, as~\Cref{fig:univariate_large_sample} shows, at large $n$ the error of the Gaussian mechanism is below that of the median (and thus also below that of BoxEM).
In these experiments, this crossover point occurs around $n=1,300$ samples.

\begin{figure}[t]
\centerline{
\includegraphics[width=0.8\columnwidth]{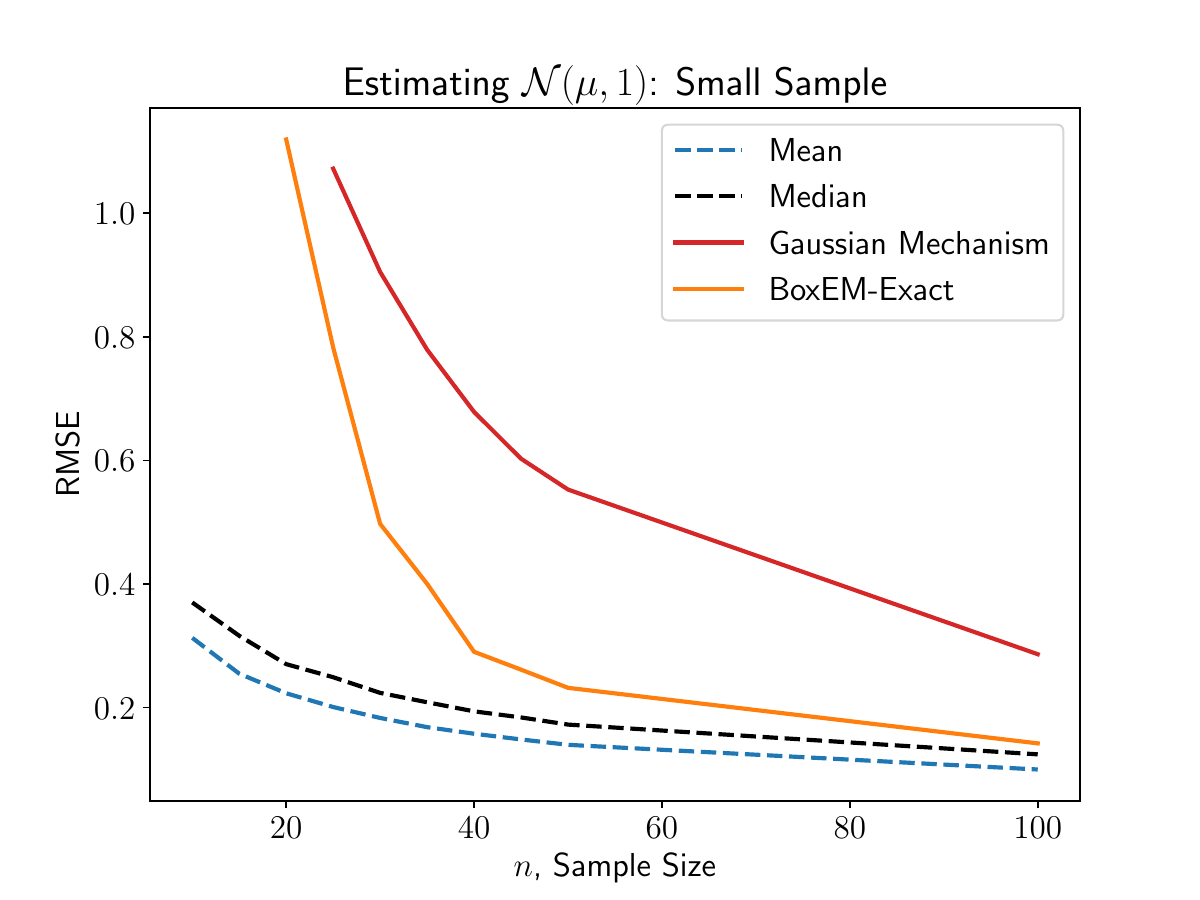}}
\caption{Even at small sample sizes, the exponential mechanism over quantiles has error comparable to that of the (nonprivate) median, and lower than that of the Gaussian mechanism.}
\label{fig:univariate_small_sample}
\end{figure}

\begin{figure}[t]
\centerline{
\includegraphics[width=0.8\columnwidth]{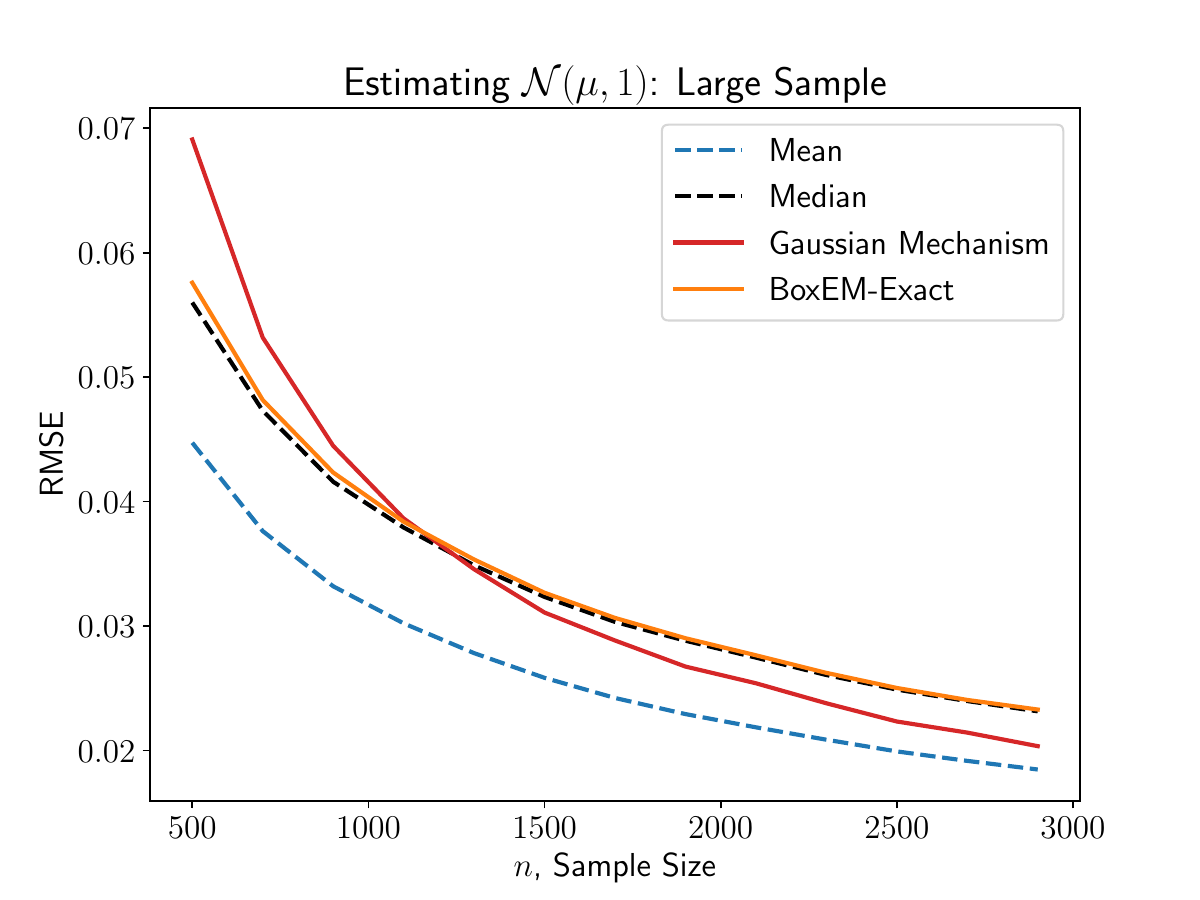}}
\caption{For large $n$, the error of the Gaussian mechanism approaches that of the empirical mean. Eventually, it is lower than the error of the median, and thus also outperforms  BoxEM.}
\label{fig:univariate_large_sample}
\end{figure}

\editblockdone

\editblock
\section{Accuracy Guarantees of BoxEM}\label{sec:BoxEM_accuracy}
In this section, we show that BoxEM's accuracy guarantees are slightly better than stated in~\cite{liu2021robust}. Overall, the proof follows closely the proof of Theorem 3.2 in BGSUZ and the relevant steps from the proof of Theorem 12 in~\cite{liu2021robust}. 
\begin{thm}\label{thm:BoxEM_accuracy}
Let $0<\alpha,\beta,\eps < 1$, $\Sigma: \id \preceq \Sigma\preceq \kappa \id$, $\mu\in[-R/2,R/2]^d$. Assume $R\gg \max\{\sqrt{\kappa}, \alpha\}$. If $x\sim \cN(\mu,\Sigma)^{\otimes n}$ and
    \begin{equation}
        n \ge C\left(\frac{d + \log (\frac{1}{\beta})}{\alpha^2}+ \frac{d + \log (\frac{1}{\alpha\eps\beta})}{\alpha\eps} + \frac{d\log(\frac{dR}{\alpha})}{\eps}\right),
    \end{equation}
    then with probability at least $1-\beta$, the Tukey Depth Mechanism (BoxEM) over the euclidean ball of radius $R$ returns $\hat{\mu}$ such that $\|\hat{\mu}-\mu\|_\Sigma \le O(\alpha)$.
\end{thm}
\begin{proof}
By~\cite[Lemma 3.5]{brown2021covariance}, there exists a constant $c$ such that for any $\alpha,\beta>0$ if 
\begin{equation}\label{eq:sc_typicality}
n \ge c\left(\frac{d + \log (1/\beta)}{\alpha^2}\right),
\end{equation}
if $x\sim \cN(\mu,\Sigma)^{\otimes n}$ then with probability at least $1-\beta$ for all $y\in\mathbb{R}^d$, $|T_x(y)/n-T_{\cN(\mu,\Sigma)}(y)|\le \alpha$. We condition on this event for the rest of the proof. 

Assume for simplicity that $n/4$ is an integer. Let $\mathtt{VERYBAD}$ be the set of points $y: T_x(y)\le \frac{n}{4}$.  We write $\Pr[y\in\mathtt{VERYBAD}]$ as a sum:
    \begin{align}\label{eq:prob_verybad}
        \Pr[y\in\mathtt{VERYBAD}] & = \sum_{\ell=0}^{n/4} \Pr\left[T_x(y)= \ell \right] \nonumber \\
        & \le \frac{\sum_{\ell=0}^{n/4} \Pr\left[T_x(y)= \ell \right]}{\Pr\left[T_x(y)\ge \frac{3n}{10} \right]} \nonumber\\
        & \le \frac{\sum_{\ell=0}^{n/4}V_{=\ell}\cdot e^{\eps\ell/2}}{V_{\ge \frac{3n}{10}}\cdot e^{3\eps n/20}} \nonumber\\
        & \le \frac{\min\{V_{\ge0}, (2R)^d\}\cdot e^{\eps n/8}}{V_{\ge \frac{3n}{10}}\cdot e^{3\eps n/20}} \nonumber \\
        & \le \frac{(2R)^d}{V_{\ge 3n/10}}\cdot e^{-\eps n/40}
    \end{align}

    By~\cite[Proposition 3.3]{brown2021covariance}, $T_{\cN(\mu,\Sigma)}(y)=\Phi(-\|y-\mu\|_\Sigma)$. Using this property and following the same steps as~\cite[Lemma J.2]{liu2021robust}, we get that $T_{\cN(\mu,\Sigma)}(y)\ge \frac{1}{2}-\frac{1}{\sqrt{2\pi}}\|y-\mu\|_\Sigma$. By our condition above, this implies $T_{x}(y)\ge n\left(\frac{1}{2}-\frac{1}{\sqrt{2\pi}}\|y-\mu\|_\Sigma-\alpha\right) \ge n\left(\frac{1}{2}-\frac{1}{\sqrt{2\pi}}\|y-\mu\|_2-\alpha\right)$, where the last inequality holds since $\id\preceq\Sigma$. 

    Thus for any point $y$ if $\|y-\mu\|_2\le \sqrt{2\pi}\alpha$ then $T_{x}(y)\ge n\left(\frac{1}{2}-2\alpha\right)\ge 3n/10$, by our assumption that $\alpha<1/10$. This implies that the volume $V_{\ge 3n/10}$ is at least that of a ball of radius $\sqrt{2\pi}\alpha$, which is $\frac{(\sqrt{2}\pi\alpha)^d}{\Gamma(d/2+1)}$.  Therefore Eq.~\eqref{eq:prob_verybad} becomes 
    \begin{align*}
        \Pr[y\in\mathtt{VERYBAD}] \le \left(\frac{2R}{\sqrt{2\pi}\alpha}\right)^d \cdot \Gamma(d/2+1) \cdot e^{-\frac{\eps n}{40}}
    \end{align*}
Recall that $\Gamma(d/2+1)=(d/2)!\le d^d$. We conclude that the probability $\Pr[y\in\mathtt{VERYBAD}]\le \beta$, as long as 
\begin{equation}\label{eq:sc_verybad}
    n \ge c'\left(\frac{d\log(dR/\alpha) + \log (1/\beta)}{\eps}\right),
\end{equation} 
for some constant $c'$.
Now, let $\mathtt{BAD}$ be the set of points $y: n/4 < T_x(y)\le \alpha_2<4\alpha$. Following the proof of Lemma 3.10~\cite{brown2021covariance} exactly, we get that $\Pr[y\in\mathtt{BAD}]\le \beta$, as long as 
\begin{equation}\label{eq:sc_bad}
    n \ge c''\left(\frac{d + \log (1/\alpha\eps\beta)}{\alpha\eps}\right),
\end{equation} 
for some constant $c''$.

Putting all conditions together, we conclude that if  
\begin{equation*}
    n \ge C\left(\frac{d + \log (\frac{1}{\beta})}{\alpha^2}+ \frac{d + \log (\frac{1}{\alpha\eps\beta})}{\alpha\eps} + \frac{d\log(\frac{dR}{\alpha}) + \log (\frac{1}{\beta})}{\eps}\right)
\end{equation*}
then  $\Pr[T_x(y)>\frac{1}{2}-4\alpha]\le 3\beta$. 
By our condition, this implies with the probability at least $1-3\beta$,
\begin{equation}
        T_{\cN(\mu,\Sigma)}(y) \ge \frac{1}{2} -\alpha -4\alpha = \frac{1}{2} - 5\alpha.
\end{equation}
Recall $\Phi\left(-\|y-\mu\|_{\Sigma}\right) =T_{\cN(\mu,\Sigma)}(y)$.
By definition, $\Phi(-z)=\frac{1}{2}-\frac{1}{2}\mathrm{Erf}\left(\frac{z}{\sqrt{2}}\right)$. It is easy to see that $\mathrm{Erf}(x)\geq \mathrm{Erf}(1)\cdot x\geq 0.84x$ for $x\in[0,1]$. It follows that 
\begin{equation*}
    \Phi\left(-z\right)\leq \frac{1}{2}-\frac{0.84z}{2\sqrt{2}}.
\end{equation*}
Combining the above inequalities, we have that $\|y-\mu\|_{\Sigma}\le \frac{10\sqrt{2}}{0.84}\alpha\leq 17\alpha$.
\end{proof}

In comparison, ~\cite{liu2021robust} prove that the sample complexity of BoxEM is $n \ge C\left(\frac{d + \log (\frac{1}{\beta})}{\alpha^2}+ \frac{d\log(\frac{dR}{\alpha\beta})}{\alpha\eps}\right)$. The improved sample complexity above ``decouples'' the $\log(dR)$ factor from the $1/\alpha$ factor.  

\end{document}